\documentclass[10pt,twocolumn,letterpaper]{article}
\usepackage{amsthm}
\usepackage{thmtools}

\newtheorem{theorem}{Theorem}
\usepackage[table]{xcolor}
\usepackage{thm-restate}
\usepackage{booktabs}
\usepackage{svg}

\usepackage{graphicx}
\usepackage{algorithm}
\usepackage{algpseudocode}
\usepackage[table]{xcolor}
\usepackage{xcolor}
\usepackage{amsfonts}
\usepackage{amsmath,amssymb}
\usepackage{multirow}
\usepackage{tcolorbox}
\usepackage{adjustbox}
\tcbset{
  ariaBox/.style={
    colback=blue!3!white,
    colframe=blue!60!black,
    boxrule=0.6pt,
    arc=3pt,
    left=8pt,
    right=8pt,
    top=8pt,
    bottom=8pt
  }
}
\newcounter{infthm}

\newcommand{\cluster}{\mathcal{C}}
\newcommand{\std}[1]{\,\raisebox{0.15ex}{$\scriptscriptstyle\pm#1$}}
\usepackage[pagenumbers]{wacv}      

\definecolor{wacvblue}{rgb}{0.21,0.49,0.74}
\usepackage[pagebackref,breaklinks,colorlinks,allcolors=wacvblue]{hyperref}

\def\wacvPaperID{1400} 
\def\confName{WACV}
\def\confYear{2027}

\title{ARIA: Adaptive Region-Based Importance Allocation for Conditional Diffusion Distillation}

\author{Loay Mualem\textsuperscript{1,2} \quad
Vinh Tong\textsuperscript{1} \quad
Samir Darouich\textsuperscript{1,3} \quad
Mathias Niepert\textsuperscript{1}\\[0.5em]
\textsuperscript{1}Institute for AI, University of Stuttgart \quad
\textsuperscript{2}IMPRS-IS\\
\textsuperscript{3}Institute of Theoretical Chemistry, University of Stuttgart.\\
[0.5em]
\textbf{Please send any questions to:}
\href{mailto:loaymua@gmail.com}{\texttt{loaymua@gmail.com}}
}

\begin{document}
\maketitle

\begin{abstract}
Distilling conditional diffusion models aims to transfer the behavior of a large teacher to a smaller student while preserving alignment across conditioning inputs. Unlike recognition tasks, knowledge distillation in conditional diffusion often struggles to transfer knowledge beyond the training distribution, since the predicted noise strongly depends on the conditioning signal. As a result, effective distillation requires exploring a large conditioning space. 
In practical settings, this creates a major bottleneck. Paired image–condition data may be limited, and generating synthetic images for every available condition is often computationally infeasible, while the pool of conditions, such as text prompts, can be extremely large. Recent work addresses this issue by switching conditions during training, exposing the student to a broader conditioning space without changing the distillation objective. Yet this raises a complementary question: once a large conditioning corpus is available, how should the training effort be allocated? In this work, we introduce ARIA, a framework that adaptively allocates training effort across coarse regions of the conditioning space. By maintaining online estimates of teacher–student discrepancy at the region level, ARIA focuses updates where misalignment persists while preserving the original distillation objective. Empirically, ARIA improves over RC across most architectures and settings, with the clearest gains observed in unseen and underrepresented regimes. We also provide a theoretical analysis showing that the proposed tracking mechanism follows the evolving discrepancy during training under bounded variance and drift assumptions.

\end{abstract}
\section{Introduction}

Diffusion models have achieved strong generative performance across many domains, including images~\cite{ho2020denoising, song2020score, dhariwal2021diffusion, batifol2025flux, rombach2022high, nichol2021glide}, video~\cite{ho2022imagen, gupta2024photorealistic, blattmann2023stable}, audio~\cite{kong2020diffwave, liu2023audioldm}, and robotics~\cite{chi2025diffusion,carvalho2025motion,chi2025diffusion}. These models iteratively transform noise into samples from a target distribution. Large-scale text-to-image systems such as Stable Diffusion~\cite{Rombach_2022_CVPR} generate images that follow natural-language prompts, however, their strong performance typically requires large models and many sampling steps, motivating efforts to develop more efficient variants.

Knowledge distillation (KD) is a widely used approach for compressing large models by training a smaller student to mimic the predictions or intermediate representations of a teacher~\cite{hinton2014distilling,yang2021knowledge,zhao2022decoupled,zhou2021rethinking, yin2024one, luo2024one, xie2024distillation, zhou2024score, sauer2024adversarial, nguyen2024swiftbrush, dao2024swiftbrush}. In text-to-image diffusion models, this typically involves sampling image–text pairs, generating noisy intermediates along the diffusion trajectory, and training the student to match the teacher’s outputs at each timestep~\cite{hinton2014distilling,yang2021knowledge,zhao2022decoupled,zhou2021rethinking} or feature representations~\cite{yang2021knowledge,zhao2022decoupled,zhou2021rethinking,chen2021distilling,li2022knowledge}. While KD can transfer knowledge beyond the training distribution in recognition models~\cite{chen2021distilling,li2022knowledge}, this effect is weaker in conditional diffusion models~\cite{kim2025random}, where the predicted noise is strongly condition-dependent. As a result, effectively distilling the teacher often requires exploring a large portion of the conditioning space.

In practice, access to large-scale paired image–text datasets is often limited due to copyright, privacy, and licensing constraints. Generating synthetic images with the teacher model is a possible alternative, but large-scale diffusion sampling is computationally expensive and storage-intensive~\cite{zhou2024score, sauer2024adversarial}. In contrast, textual data is abundant: large collections of prompts and captions are widely available at negligible cost. This creates a fundamental imbalance—while the conditioning space (text) can be extremely large, the number of available images for distillation is typically constrained by computational resources.



To address this mismatch, Kim et al.~\cite{kim2025random} proposed Random Conditioning (RC), which pairs noisy images with randomly sampled text conditions, enabling exploration of the conditioning space without generating images for every prompt.
RC answers an important feasibility question: can auxiliary text improve diffusion distillation when paired image--prompt data is limited?
Our work asks the complementary allocation question: once a large auxiliary text pool is available, how should the training budget be distributed over it?
A static coverage strategy treats the auxiliary conditioning space as a fixed resource and samples it uniformly.
However, the usefulness of a region is student-dependent and changes during training, as some regions are learned quickly, whereas others continue to exhibit a large teacher--student gap. To this end we introduce ARIA (\textit{Adaptive Region-based Importance Allocation}), which uses auxiliary text not only to expand coverage, but to provide adaptive feedback. ARIA groups auxiliary conditions into coarse regions, tracks their teacher--student discrepancy online, and reallocates sampling toward regions where supervision is currently most needed. 


\begin{figure*}[t]
    \centering \includegraphics[width=\textwidth,keepaspectratio]{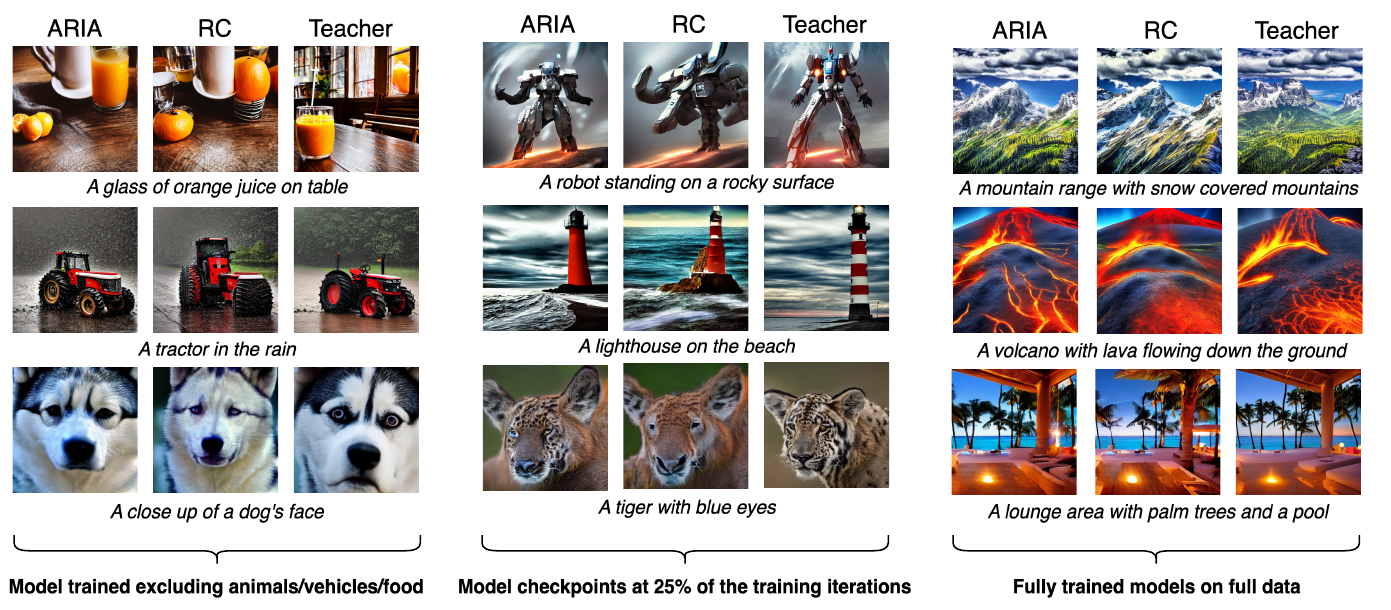} 
    \caption{Side-by-side comparison of images generated by ARIA, RC, and Teacher models under three settings: training without animals, vehicles, or food; using $25\%$ training checkpoints; and fully trained on the complete dataset.
    }
    \label{fig:aria_comparison}
\end{figure*}

Concretely, ARIA is a lightweight replacement for the condition-selection step in RC. The paired-image cache, teacher--student objective, and model architectures remain unchanged, while the auxiliary prompt sampler becomes discrepancy-aware. We instantiate ARIA for text-conditioned diffusion model distillation in the regime of limited paired images and abundant auxiliary text. ARIA leaves the underlying distillation objective unchanged and only modifies how conditions are sampled. We provide theoretical analysis of the region-level discrepancy tracking mechanism and demonstrate that ARIA improves over random conditioning across most architectures and data regimes, with the strongest gains in unseen and underrepresented settings, and faster convergence.

\textbf{Contributions.} Our main contributions are as follows:
(1) We introduce ARIA, a region-based framework for adaptive importance allocation that prioritizes coarse condition regions with higher discrepancy. Unlike per-sample importance sampling, ARIA tracks discrepancy at the level of coarse condition regions, enabling scalable adaptive allocation over extremely large input spaces without per-sample scoring.
(2) We instantiate ARIA for text-to-image diffusion distillation with limited paired images and abundant auxiliary text, achieving consistent improvements over random sampling, particularly for unseen prompts and imbalanced conditioning distributions.
(3) We present extensive empirical evaluation demonstrating stable training behavior and robustness across architectures and region constructions including SD~1.4, SD~2.1, and SDXL teachers with block-pruned, channel-pruned, and KOALA students.
(4) We provide theoretical analysis of ARIA’s EMA-based region scoring rule, establishing finite-time guarantees for tracking evolving discrepancy under noise and temporal drift.

\textbf{Organization.} Section~\ref{sec:related_work} reviews related work, and Section~\ref{sec:problem_setting} formalizes the distillation setting and RC baseline.
Section~\ref{sec:aria} introduces ARIA and presents its algorithmic formulation together with the tracking analysis.
Section~\ref{sec:exp} applies ARIA to text-to-image distillation under two experimental setups, demonstrating performance improvements and robustness.
We conclude in Section~\ref{sec:conclusion} with a discussion and future directions.

\section{Related Work}

\label{sec:related_work}
\textbf{Compressed Diffusion Models.}
Modern diffusion models are typically large and require substantial computational and memory resources for both training and inference. Several works aim to reduce the computational and memory cost of diffusion models through model compression. 
Existing approaches include quantization~\cite{shang2023post}, architecture optimization~\cite{li2023snapfusion}, pruning~\cite{guo2025mosaicdiff,kwon2025hierarchicalprune,zhu2024dip,xie2025sana,chen2025sana} and knowledge distillation~\cite{xiang2025dkdm,kim2025random,kim2024bk}. 
For example, BK-SDM~\cite{kim2024bk} applies knowledge distillation together with block pruning to compress Stable Diffusion~\cite{Rombach_2022_CVPR}, while KOALA~\cite{lee2024koala} performs layer-wise compression and knowledge distillation for SDXL~\cite{podellsdxl}.

\textbf{Knowledge Distillation for Model Compression.}
Knowledge distillation (KD)~\cite{hinton2014distilling} is widely used for training compact models by transferring knowledge from a larger teacher through soft predictions or intermediate representations~\cite{yang2021knowledge,zhao2022decoupled,zhou2021rethinking,chen2021distilling,li2022knowledge}. 
KD has been successfully applied across many domains, including language models, vision transformers, and diffusion models~\cite{sun2019patient,jiao2020tinybert,hao2022learning,touvron2021training,kim2024bk,kim2025random,xiang2025dkdm}. 
In diffusion models, KD is commonly used to train compressed architectures. However, Kim et al.~\cite{kim2025random} show that conventional distillation strategies may struggle to transfer knowledge for underrepresented or uncovered concepts, motivating improved data allocation strategies during training. 

\textbf{Knowledge Distillation for Decreasing Sampling Steps.}
Several works focused on accelerating the denoising process in diffusion models without retraining~\cite{lu2022dpm, zhao2023unipc, zhang2022fast, zheng2023dpm}, reducing sampling steps from thousands to as few as 10–25. However, pushing this reduction further typically leads to substantial performance degradation. Distillation-based acceleration methods~\cite{salimans2022progressive, song2023consistency, tong2024learning, dao2024swiftbrush, nguyen2024swiftbrush, zhou2024score, luo2024one, xie2024distillation, yin2024one} address this by training student models to compress multi-step denoising trajectories into fewer steps, sometimes even a single step. Importantly, these methods focus on reducing sampling steps rather than compressing model capacity. In contrast, our work targets the compression of the base diffusion model itself, providing a compact foundation that can naturally complement and strengthen step-acceleration techniques.

\section{Problem Setting}\label{sec:problem_setting}
We study the distillation of a pretrained text-to-image diffusion model (the \textit{teacher}) into a smaller student model as presented in~\cite{kim2025random,kim2024bk}.
For clarity, we use the noise prediction objective, though the same derivation applies to score, velocity, or data prediction.
Let $\epsilon_\mathcal{T}(x_t, t, c)$ denote the teacher and $\epsilon_\mathcal{S}(x_t, t, c)$ the student, where $x_t$ is a noisy latent at timestep $t$ and $c$ is a text condition.

\textbf{Conditional Distillation Loss.}
A straightforward strategy for image-free distillation is to first synthesize
images conditioned on text prompts and construct a paired dataset
\(
\mathcal{D} = \{(\mathbf{x}^n, c^n)\}_{n=1}^{N}
\),
where \(\mathbf{x}^n\) denotes the generated image corresponding to the text
condition \(c^n\).
The generated image serves as the original clean sample \(\mathbf{x}_0\),
from which we can produce a noisy input \(\mathbf{x}_t\) for any timestep
\(t\) under condition \(c^n\).
Since diffusion models require substantial computational cost for image
generation, these synthetic images are typically generated and stored in
advance to form the training dataset.
The teacher model can subsequently be distilled into a student model using
the following objective:

\begin{equation}
\mathcal{L}_{\mathrm{out}}
=
\mathbb{E}_{(\mathbf{x}_t, c)\in \mathcal{D},\, t}
\left[
\left\|
\epsilon_{\mathcal{T}}(\mathbf{x}_t, c, t)
-
\epsilon_{\mathcal{S}}(\mathbf{x}_t, c, t)
\right\|_2^2
\right],
\tag{1}
\end{equation}

where \(\epsilon_{\mathcal{T}}\) and \(\epsilon_{\mathcal{S}}\) represent the
noise predictions of the teacher and student networks, respectively.
Here, \((\mathbf{x}_t, c)\) is sampled from the dataset \(\mathcal{D}\),
with noise injected according to timestep \(t\), which is drawn uniformly
from the interval \([0, T]\).

Furthermore, one may augment the training objective with a feature-level
distillation term defined as

\begin{equation}
\mathcal{L}_{\mathrm{feat}}
=
\mathbb{E}_{(\mathbf{x}_t, c)\in \mathcal{D},\, t}
\left[
\sum_{l}
\left\|
f_{\mathcal{T}}^{l}(\mathbf{x}_t, c, t)
-
f_{\mathcal{S}}^{l}(\mathbf{x}_t, c, t)
\right\|_2^2
\right],
\tag{2}
\end{equation}

where \(f_{\mathcal{T}}^{l}\) and \(f_{\mathcal{S}}^{l}\) denote the feature
representations extracted at layer \(l\) from the teacher and student models,
respectively. 

\textbf{Random Conditioning.}
Let $\mathcal{C}_{\text{aux}} = \{c_j\}_{j=1}^M$ denote a large corpus of text conditions. 
Generating paired images for all conditions requires repeated diffusion sampling and is therefore computationally infeasible. 
Instead, training relies on a smaller cached dataset 
$\mathcal{D} = \{(\mathbf{x}^n, c^n)\}_{n=1}^N$ with $N \ll M$. 
Kim et al.~\cite{kim2025random} showed that both sources can be leveraged via \emph{Random Conditioning (RC)}, which replaces the condition associated with a cached image during training. 
Specifically, RC samples $(\mathbf{x}^n, c^n) \sim \mathcal{D}$ and a timestep $t \sim \mathrm{Unif}([0,T])$, constructs $\mathbf{x}_t$ via the forward process, and sets the training condition
\begin{equation}
\hat c =
\begin{cases}
c^n, & \text{with probability } 1 - p(t),\\
\tilde c \sim \mathrm{Unif}(\mathcal{C}_{\text{aux}}), & \text{with probability } p(t),
\end{cases}
\tag{3}
\end{equation}
where $p(t)$ is a timestep-dependent schedule.

Uniform sampling enables broad exploration of $\mathcal{C}_{\text{aux}}$, but implicitly assumes that all regions are equally informative. 
As training progresses, discrepancies across conditions become heterogeneous: some regions quickly align while others remain difficult. 
Uniform exploration therefore wastes iterations on already-aligned regions and under-allocates effort to persistently misaligned ones.

\begin{figure*}[htb]
    \centering
     \includegraphics[width=0.90\textwidth,keepaspectratio]{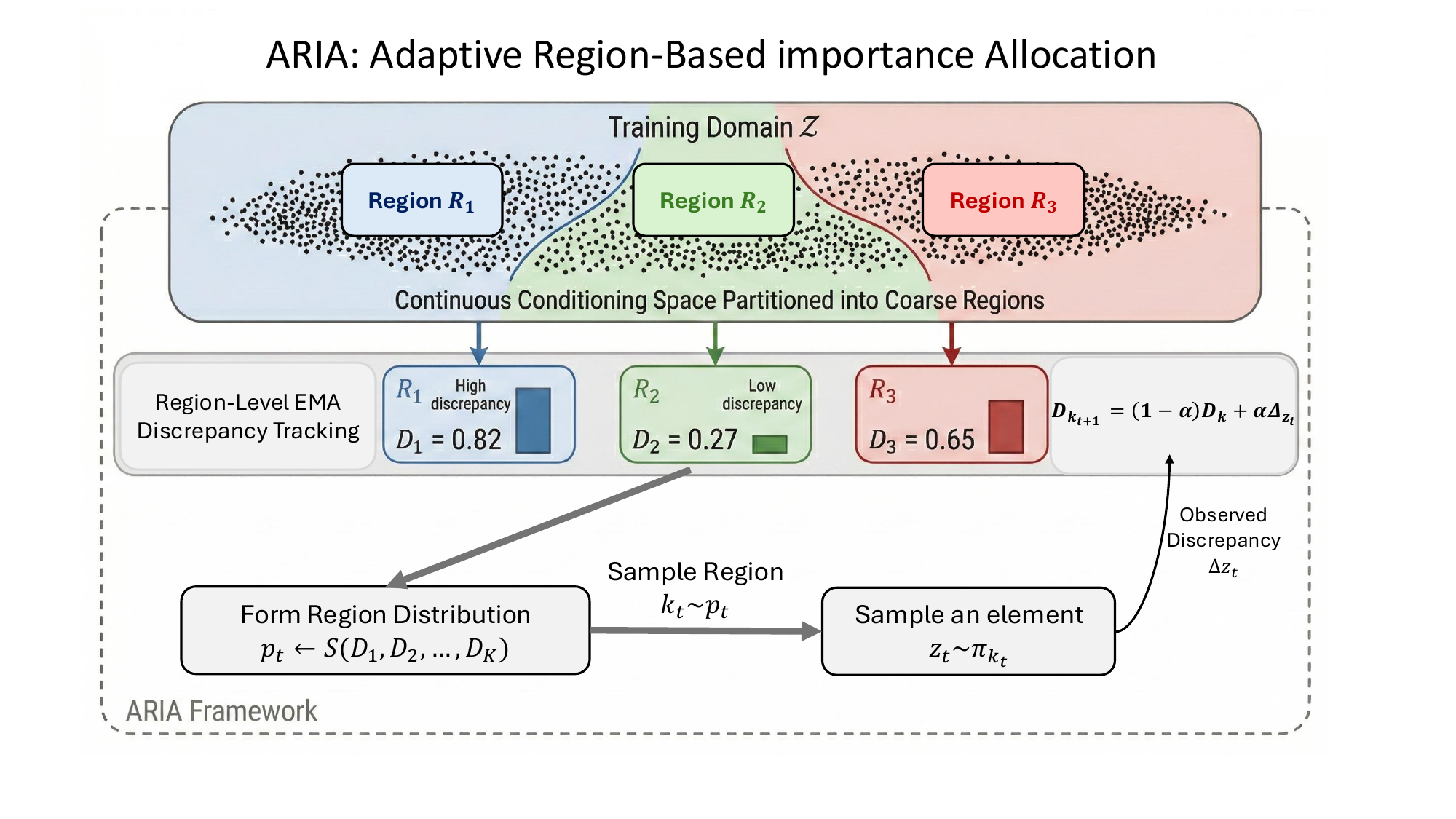} 
    \caption{\textbf{Overview of ARIA.}
The conditioning domain $\mathcal{Z}$ is partitioned into coarse regions $R_k$. 
ARIA maintains an EMA discrepancy score $D_k$ per region and converts these scores into a sampling distribution $p_t=\mathcal{S}(D_1,\dots,D_K)$. 
At each step, a region $k_t\sim p_t$ and a sample $z_t\sim\pi_{k_t}$ are drawn, the discrepancy $\Delta(z_t)$ is observed, and the corresponding score is updated via EMA.}
    \label{fig:aria_overview}
\end{figure*}

\section{ARIA: Adaptive Region-based Importance Allocation}
\label{sec:aria}

In what follows, we introduce ARIA (Algorithm~\ref{alg:aria}, Figure \ref{fig:aria_overview}), our approach for adaptive importance allocation that operates at the level of coarse regions in the input space. By tracking discrepancy statistics at the region level rather than per sample, ARIA enables adaptive reallocation of training effort toward persistently misaligned areas while remaining fully compatible with the underlying learning objective. For clarity, we first describe ARIA in a general form and then instantiate it for diffusion distillation in Sec.~\ref{sec:exp}.

ARIA assumes that the training domain is partitioned into coarse regions.
These regions may be defined in any user-specified representation (e.g., clustering in embedding space, binning in latent space, or grouping by metadata).
During training, ARIA maintains a scalar discrepancy score for each region using an exponential moving average (EMA) and biases sampling toward regions with larger tracked discrepancy.
Importantly, ARIA does not modify the underlying learning objective; it only alters how training samples are selected.

\subsection{Regions and discrepancy signal}
\label{sec:aria_regions}

Let $\mathcal{Z}$ denote the set of training samples relevant to the learning objective.
In this work, $\mathcal{Z}$ corresponds to the auxiliary prompt set $\mathcal{C}_{\text{aux}}$ introduced in Sec.~\ref{sec:problem_setting}.
ARIA assumes a partition of $\mathcal{Z}$ into $K$ coarse regions of the training domain
\[
\mathcal{R}=\{R_k\}_{k=1}^K,\quad
\cup_{k=1}^K R_k=\mathcal{Z},\quad
R_i\cap R_j=\emptyset\ (i\neq j).
\]
This partition is used as a computational device for maintaining region-level statistics, not as an assumption that the underlying factors of variation are mutually exclusive. Individual samples may involve multiple factors, but assigning each sample to a single region enables scalable discrepancy tracking without per-sample scoring. Regions can be constructed in any representation space. In our text-to-image instantiation, they are obtained by clustering prompt embeddings; details are provided in Sec.~\ref{sec:exp}.
To quantify the learning signal associated with a sampled training instance,
ARIA relies on a scalar \emph{discrepancy} function $\Delta(\cdot)$.
At iteration $t$, $\Delta_t$ denotes the scalar value used to update the score of the selected region,
e.g., the minibatch-averaged loss from samples drawn from that region.

\subsection{Region score initialization and tracking}
\label{sec:aria_ema}

For each region $R_k$, ARIA maintains a score $D_k$ intended to approximate its current mean discrepancy.

\paragraph{Initialization.}
Before adaptive sampling begins, we estimate an initial score for each region by evaluating the discrepancy on a small random subset.
Specifically, for each $k$, we sample a batch $\mathcal{B}_k \subset R_k$ and set
\begin{equation}
\label{eq:aria_init}
D_k^{(0)} \;=\; \frac{1}{|\mathcal{B}_k|}\sum_{z\in\mathcal{B}_k}\Delta(z).
\end{equation}
This provides a coarse estimate of which regions are initially under-aligned.

\paragraph{Online tracking via EMA.}
Whenever region $k$ is selected at training step $t$ and discrepancy $\Delta_t$ is observed,
its score is updated using an exponential moving average:
\begin{equation}
\label{eq:aria_ema_update}
D_k \leftarrow (1-\alpha)D_k + \alpha\,\Delta_t,
\qquad \alpha\in(0,1).
\end{equation}
All other region scores remain unchanged at that step.
EMA provides a stable tracker in the presence of stochasticity and non-stationarity.
Our theoretical analysis in subsection ~\ref{subsection:alg_analysis} quantifies its tracking behavior.

\subsection{Score-based region sampling}
\label{sec:aria_sampling}
At each training step $t$, ARIA forms a probability distribution over regions based on the current scores $\{D_k\}_{k=1}^K$.
Formally, let
\[
p_t = \mathcal{S}(D_1,\dots,D_K),
\]
where $\mathcal{S}:\mathbb{R}^K \to \Delta^{K-1}$ maps region scores to a probability distribution over $\{1,\dots,K\}$.

The only requirement is that regions with larger scores receive larger sampling probability. In practice, the mapping should balance prioritization of high-discrepancy regions with sufficient coverage of the remaining conditioning space.
Beyond this monotonicity condition, the mapping $\mathcal{S}$ can be chosen according to practical considerations. In our experiments, we instantiate $\mathcal{S}$ using a softmax transformation of the region scores.

Each region $R_k$ is associated with a sampler $\Pi_k$ that draws training instances from that region.
In this work, we instantiate $\Pi_k$ as uniform sampling over elements of $R_k$.
More generally, $\Pi_k$ may implement any user-defined sampling strategy within the region.
\begin{algorithm}[htb!]
\caption{ARIA: Adaptive Region-based Importance Allocation}
\label{alg:aria}
\begin{algorithmic}[1]
\Require Regions $\mathcal{R}=\{R_1,\dots,R_K\}$ with samplers $\{\Pi_k\}$,
discrepancy $\Delta(\cdot)$, EMA parameter $\alpha\in(0,1)$.
\Require Sampling mapping $\mathcal{S}:\mathbb{R}^K \to \Delta^{K-1}$ (maps scores to a distribution).
\Require Initialization fraction $\rho\in(0,1)$.
\State \textbf{(Initialization)} For each region $k$, sample $\mathcal{B}_k \subset R_k$ and set
$D_k^{(0)} \leftarrow \frac{1}{|\mathcal{B}_k|}\sum_{z\in\mathcal{B}_k}\Delta(z)$.
\For{$t=1,2,\dots,T$}
    \State Form region distribution $p_t \leftarrow \mathcal{S}(D_1,\dots,D_K)$.
    \State Sample region $k_t \sim p_t$.
    \State Sample instance $z_t \sim \Pi_{k}$.
    \State Compute discrepancy $\Delta_t \leftarrow \Delta(z_t)$ and perform the standard learning update.
    \State Update the selected region score:
    \[
        D_{k_t} \leftarrow (1-\alpha)D_{k_t} + \alpha\,\Delta_t.
    \]
\EndFor
\end{algorithmic}
\end{algorithm}

Given $p_t$, a region index $k_t \sim p_t$ is sampled and a training instance $z_t \sim \Pi_{k_t}$ is drawn from the corresponding region.
The discrepancy $\Delta(z_t)$ is then computed and used both for the learning update and for updating the EMA score of region $k_t$.
ARIA therefore modifies only the data-selection policy while leaving the underlying learning objective unchanged.

\subsection{Algorithm and tracking analysis}
\label{subsection:alg_analysis}

Sampling probabilities depend on region-level scores, thus these must reflect the evolving discrepancy. We show that ARIA’s EMA scores track the conditional mean discrepancy with a finite-time guarantee whose error decomposes into three terms: an initialization term, a drift-induced bias, and a stochastic term that scales with $\sqrt{\alpha}$. 
In addition, a high-probability bound ensures stability over finite horizons.

These guarantees support the use of EMA scores as lightweight proxies for region difficulty under bounded variance and bounded drift assumptions.

For brevity, we present an informal statement below (Theorem~\ref{thm:informal}) and refer to Sec.~B in the Supp. Mat. for the full theorem and proof.

\begin{theorem}[Tracking Guarantee, Informal]
\label{thm:informal}
Fix a region $s$ and let $D_s^{(n)}$ denote ARIA's EMA score after its $n$-th update.
Let $J_s^{(n)} := \mathbb{E}[\Delta_s^{(n)} \mid \mathcal{F}_{n-1}]$ denote the expected discrepancy for region $s$, where the expectation is taken over the within-region sampler (uniform over $R_s$ in this work).
Define the tracking error $e_n := D_s^{(n)} - J_s^{(n)}$.
Assume bounded conditional variance $\sigma^2$ and bounded discounted drift $V$.
Then for all $n \ge 1$,
\[
\sqrt{\mathbb{E}[e_n^2]}
\;\le\;
(1-\alpha)^n |e_0|
+
(1-\alpha)V
+
\sigma \sqrt{\alpha}.
\]

The error decomposes into an exponentially decaying initialization term,
a drift-induced bias term, and a stochastic term scaling with $\sqrt{\alpha}$.
A uniform high-probability bound further ensures stability over finite horizons.
\end{theorem}

\section{Experiments}
\label{sec:exp}
In this section, we evaluate ARIA across multiple student architectures and data regimes.
Following~\cite{kim2025random}, we use the exponential schedule
$p(t) = e^{-\lambda \left(1 - \frac{t}{T}\right)}$.
As the discrepancy signal $\Delta(\cdot)$, we use the output-level distillation loss
$\mathcal{L}_{\mathrm{out}}$.

For each region $R_k$ in the conditioning space, ARIA tracks the expected output-level loss
\[
J_k := \mathbb{E}_{c \in R_k,\, x_t}\left[\mathcal{L}_{\mathrm{out}}(x_t,c)\right],
\]
where $\mathcal{L}_{\mathrm{out}}(x_t,c)$ denotes the denoising loss evaluated for condition $c$
and noisy latent $x_t$ sampled from the forward diffusion process.
In practice, this quantity is estimated online using an exponential moving average over sampled instances from $R_k$. 
Specifically, the score of region $R_k$ is updated whenever a condition $c \in R_k$ is selected as a replacement during training, using the observed distillation loss.
We use a softmax mapping for $\mathcal{S}$ and set the EMA parameter to $\alpha=0.1$.


To construct the regions, we extract prompt embeddings from the auxiliary corpus using a pretrained CLIP~\cite{radford2021learning} text encoder and apply k-means clustering in the embedding space.
The number of clusters is selected using the Silhouette score to balance inter and intra-cluster cohesion.
Further details on the clustering procedure are provided in Sec.~\ref{sec:clustering} of the Supp. Mat., while sensitivity to these design choices is analyzed in Sec.~\ref{sec:ablation} and Sec.~\ref{sec:supp_additional_eval} of the Supp. Mat.

\paragraph{Roadmap.}
Our evaluation is structured around two complementary scenarios.
First, in subsection~\ref{exp:setup} we describe the experimental setup.
Next, in subsection~\ref{sec:exp_balanced} we consider a balanced setting that reproduces the standard knowledge distillation setup as in~\cite{kim2024bk,kim2025random} to assess convergence and final performance.
Finally, in subsection~\ref{sec:exp_imbalanced} we construct an imbalanced concept-removal setting to test robustness under distributional gaps. Additional ablation studies of our method are provided in Section~\ref{sec:ablation}.
\subsection{Experimental Setup}\label{exp:setup}
We evaluate ARIA across three teacher models: Stable Diffusion v1.4~\cite{rombach2022high}, Stable Diffusion v2.1, and SDXL~\cite{podellsdxl}. We distill these teachers into multiple student architectures across several settings. For Stable Diffusion v1.4, we consider (i) BK-SDM students~\cite{kim2024bk}, trained for 75K steps, and (ii) channel-pruned U-Net variants~\cite{kim2025random}, trained for 300K steps\footnote{RC~\cite{kim2025random} reports results up to $400$K iterations. In our experiments, channel-pruned models saturate around $300$K, with negligible gains from further training.}. For Stable Diffusion v2.1, we distill into two block-pruned students, BK-SDM-v2-S and BK-SDM-v2-T, trained for 25K steps. For SDXL, we use KOALA-700M~\cite{lee2024koala} as the student and train for 100K steps.
Table~\ref{tab:kd_balanced_main} reports the sizes of the variants in each family.

For each model, we report the mean and standard deviation computed over three random seeds.
\paragraph{Training Dataset.}
For the distillation of Stable Diffusion v2.1, Stable Diffusion v1.4, and SDXL, we follow~\cite{kim2025random} and construct the paired training set from LAION-Aesthetics V2~\cite{schuhmann2022laion_aesthetics}, a subset of the LAION corpus~\cite{schuhmann2021laion}.
Specifically, we sample $212$K image–text pairs from the LAION-Aesthetics V2 6.5+ subset and use them as the training set.
To simulate the image-free setting, we discard the original images and retain only the associated text prompts; synthetic images are then generated from these prompts using the corresponding teacher model.
Separately, we sample $20$M text prompts from the full LAION corpus to construct the auxiliary text pool used by our method.

For the filtered setting described below, we remove training samples belonging to selected semantic categories, namely animals, vehicles, and food, using BLIP~\cite{li2022blip} together with keyword-based filtering.
After filtering, the paired training set contains approximately $168$K image–text pairs.


\paragraph{Evaluation.}
Following~\cite{kim2024bk,kim2025random}, we use $30$K image–text pairs sampled from the MS-COCO~\cite{lin2014microsoft} validation set.
Each image is associated with five human-annotated captions; following~\cite{kim2024bk}, we use a single caption from~\cite{kim2024bk} for evaluation. We report the following evaluation metrics:
Fr\'echet Inception Distance (FID, $\downarrow$),
Inception Score (IS, $\uparrow$), and
CLIP score ($\uparrow$). We use the Inception-v3 model for computing FID and IS, and the Vit-g/14 model for calculating CLIP score. We verify that none of the MS-COCO validation captions used for evaluation appear in the LAION auxiliary prompt pool.

\paragraph{Computational overhead.}
The only additional preprocessing cost in ARIA stems from embedding and clustering the auxiliary text corpus, a process that handles 20M text prompts using only $5.3$ A$100$ GPU hours, a small fraction of the $400$–$600$ GPU hours required for student training.
The computational overhead across student models is summarized in Fig.~\ref{fig:efficiency}, and during training ARIA adds negligible cost, as region scores are updated from distillation losses via lightweight EMA operations.
\begin{figure*}[t]
    \centering
    \includegraphics[width=\textwidth]{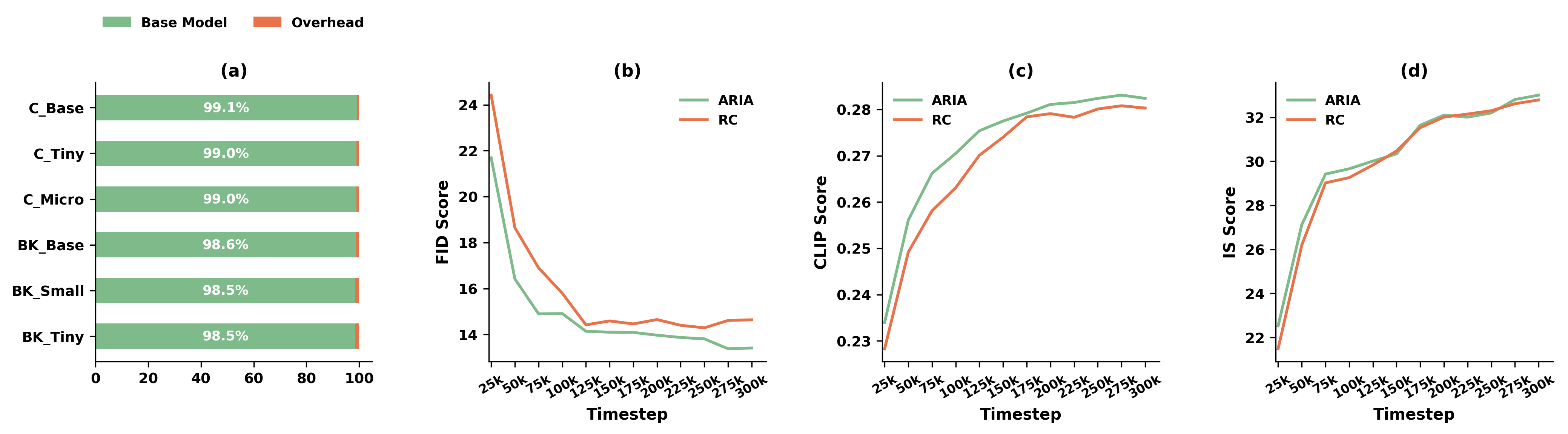}
    \caption{
    \textbf{Training efficiency and convergence comparison.}
    (a) Relative computational overhead of ARIA compared to the base model,
    showing negligible additional cost.
    (b) FID across training checkpoints.
    (c) CLIP score across training checkpoints.
    (d) IS across training checkpoints.
    ARIA converges faster than RC,
    achieving lower FID and higher CLIP at earlier stages of training.
    }
    \label{fig:efficiency}
\end{figure*}
\subsection{Balanced Distillation Setting}
\label{sec:exp_balanced}

\begin{table}[htb!]
\centering
\small
\setlength{\tabcolsep}{4pt}
\renewcommand{\arraystretch}{1.18}

\resizebox{\linewidth}{!}{%
\begin{tabular}{@{}lll|rrr|rrr@{}}
\toprule
& & & \multicolumn{3}{c|}{\textbf{RC}} & \multicolumn{3}{c}{\textbf{Ours (ARIA)}} \\
\cmidrule(lr){4-6}\cmidrule(lr){7-9}
\textbf{Family} & \textbf{Model} &\textbf{\#Params} 
& \textbf{FID}$\downarrow$ & \textbf{IS}$\uparrow$ & \textbf{CLIP}$\uparrow$
& \textbf{FID}$\downarrow$ & \textbf{IS}$\uparrow$ & \textbf{CLIP}$\uparrow$ \\
\midrule

\rowcolor{gray!10}
\multicolumn{2}{l}{\textbf{Teacher (SD 1.4)}} & 1.04B &
\multicolumn{6}{c}{\textbf{13.05} \quad / \quad \textbf{36.76} \quad / \quad \textbf{29.58}} \\
\midrule

\multirow{3}{*}{\textbf{C-based}}
& C-base  & 0.73B
& 14.45 & \textbf{34.57} & 28.73 
& \textbf{13.71} & 34.12 & \textbf{28.92} \\

& C-tiny  & 0.49B
& 14.64 & 32.79 & 28.09 
& \textbf{13.37} & \textbf{32.81} & \textbf{28.30} \\

& C-micro & 0.40B
& 14.39 & 30.20 & 27.56 
& \textbf{13.77} & \textbf{30.35} & \textbf{27.75} \\

\midrule

\multirow{3}{*}{\textbf{B-based}}
& BK-base  & 0.76B
& 15.16 & 35.90 & 29.15 
& \textbf{13.82} & \textbf{36.52} & \textbf{29.28} \\

& BK-small & 0.66B
& 16.19 & 35.34 & 27.74 
& \textbf{15.65} & \textbf{36.24} & \textbf{28.03} \\

& BK-tiny  & 0.50B
& 15.72 & 34.51 & 27.78 
& \textbf{15.16} & \textbf{35.12} & \textbf{27.91} \\

\bottomrule
\end{tabular}
}

\caption{\textbf{Balanced distillation setting.} Students are trained using $212$K cached LAION pairs and an auxiliary text pool $\mathcal{C}_{\text{aux}}$ of $20$M prompts. ARIA improves over Random Conditioning across most metrics and student architectures.}
\label{tab:kd_balanced_main}
\end{table}

We start by evaluating ARIA on the full $212$K cached LAION paired dataset and the $20$M-prompt $\mathcal{C}_{\text{aux}}$.
This regime evaluates performance under a broadly representative and semantically diverse prompt distribution.

In this setting, we study:
(i) final generative quality,
(ii) convergence speed,
and (iii) robustness across student architectures.
Since all training components remain unchanged except prompt selection,
any observed improvements can be attributed to effective allocation of training effort across the conditioning space.
\begin{table*}[htb!]
\centering
\small
\setlength{\tabcolsep}{4pt}
\renewcommand{\arraystretch}{1.15}
\begin{adjustbox}{max width=\textwidth}
\begin{tabular}{l ccc  ccc | ccc ccc}
\toprule
& \multicolumn{6}{c}{\textbf{Seen }} 
& \multicolumn{6}{c}{\textbf{Unseen}} \\

\cmidrule(lr){2-7} \cmidrule(lr){8-13}

\textbf{Model}
& \multicolumn{3}{c}{RC}
& \multicolumn{3}{c}{Ours}
& \multicolumn{3}{c}{RC}
& \multicolumn{3}{c}{Ours} \\

& FID$\downarrow$ & IS$\uparrow$ & CLIP$\uparrow$
& FID$\downarrow$ & IS$\uparrow$ & CLIP$\uparrow$
& FID$\downarrow$ & IS$\uparrow$ & CLIP$\uparrow$
& FID$\downarrow$ & IS$\uparrow$ & CLIP$\uparrow$ \\

\midrule

\rowcolor{gray!10}
\textbf{SD 1.4}
& \multicolumn{6}{c}{\textbf{14.51\std{.17} / 31.52\std{.27} / 29.38\std{.04}}}
& \multicolumn{6}{c}{\textbf{17.64\std{.37} / 25.20\std{.41} / 29.84\std{.12}}} \\

\midrule
BK-Base
& \textbf{15.41}\std{.22} & 30.85\std{.29} & 29.09\std{.03}
& 15.87\std{.17} & \textbf{31.21}\std{.32} & \textbf{29.24}\std{.06}
& 19.66\std{.18} & 24.29\std{.22} & 29.27\std{.02}
& \textbf{18.47}\std{.21} & \textbf{25.13}\std{.17} & \textbf{29.56}\std{.05} \\

BK-Small
& 15.84\std{.31} & 30.14\std{.16} & 27.80\std{.03}
& \textbf{15.41}\std{.23} & \textbf{30.95}\std{.21} & \textbf{27.82}\std{.03}
& 21.41\std{.27} & 23.36\std{.18} & 27.85\std{.06}
& \textbf{19.91}\std{.22} & \textbf{24.19}\std{.26} & \textbf{27.98}\std{.09} \\

BK-Tiny
& 18.13\std{.14} & 28.33\std{.21} & 27.41\std{.02}
& \textbf{17.19}\std{.09} & \textbf{29.94}\std{.11} & \textbf{27.67}\std{.02}
& \textbf{20.37}\std{.48} & 23.15\std{.27} & 27.40\std{.10}
& 20.39\std{.34} & \textbf{23.79}\std{.00} & \textbf{27.92}\std{.03} \\
\midrule
C-Base
& 15.91\std{.44} & 28.59\std{.10} & 28.65\std{.08}
& 15.91\std{.51} & \textbf{29.29}\std{.12} & \textbf{28.84}\std{.05}
& 19.72\std{.33} & 23.23\std{.13} & 28.41\std{.07}
& \textbf{18.51}\std{.21} & \textbf{24.38}\std{.14} & \textbf{28.84}\std{.01} \\

C-Tiny
& 16.27\std{.12} & 27.49\std{.24} & 28.25\std{.03}
& \textbf{15.97}\std{.09} & \textbf{28.13}\std{.11} & \textbf{28.25}\std{.01}
& 20.01\std{.18} & 22.81\std{.31} & 28.01\std{.10}
& \textbf{18.35}\std{.27} & \textbf{23.61}\std{.19} & \textbf{28.37}\std{.02} \\

C-Micro
& 16.01\std{.09} & 26.71\std{.22} & 27.53\std{.02}
& \textbf{15.79}\std{.21} & \textbf{27.60}\std{.18} & \textbf{27.75}\std{.03}
& 20.39\std{.14} & 22.03\std{.28} & 27.29\std{.04}
& \textbf{18.78}\std{.28} & \textbf{22.79}\std{.39} & \textbf{27.69}\std{.09} \\

\bottomrule
\end{tabular}
\end{adjustbox}
\caption{Filtered concept-removal setting on SD~1.4. We report performance on seen and unseen concepts. All students are trained without images from the removed animal, vehicle, and food categories. ARIA generally improves over RC across multiple metrics, with the largest and most consistent gains naturally appearing in the unseen regime.}
\label{tab:filtered}
\end{table*}
Table~\ref{tab:kd_balanced_main} reports results in the balanced regime.
The teacher model (Stable Diffusion v$1.4$) is fixed and therefore shown only once for reference.
Across student architectures, ARIA improves over RC in most cases, with the clearest gains appearing in FID and CLIP, while maintaining competitive or improved IS.
Notably, the relative gains are more pronounced for smaller student models,
indicating that ARIA is particularly effective in capacity-constrained settings.

To further evaluate ARIA beyond the SD~1.4 setting, we distill SDXL~\cite{podellsdxl} into KOALA-700M~\cite{lee2024koala}. We compare ARIA against RC~\cite{kim2025random} and a static region-sampling policy that uses the same regions as ARIA but samples them uniformly. This baseline isolates the effect of online adaptive allocation from the effect of region construction alone. As shown in Table~\ref{tab:policy}, ARIA outperforms both RC and static region sampling, indicating that the gains come from discrepancy-aware allocation rather than clustering alone.

\begin{table}[H]
\centering
\small
\setlength{\tabcolsep}{4pt}
\renewcommand{\arraystretch}{1.12}
\begin{adjustbox}{max width=\columnwidth}
\begin{tabular}{lccc}
\toprule
\textbf{Policy} & FID$\downarrow$ & IS$\uparrow$ & CLIP$\uparrow$ \\
\midrule

\rowcolor{gray!10}
\textbf{SDXL}
& \textbf{13.12}\std{.39}
& \textbf{35.84}\std{.33}
& \textbf{32.57}\std{.09} \\

\midrule

Random Conditioning
& 22.45\std{.41}
& 28.72\std{.22}
& 29.39\std{.06} \\

Static Region Sampling
& 21.97\std{.37}
& 29.31\std{.28}
& 29.60\std{.05} \\

ARIA
& \textbf{20.83}\std{.46}
& \textbf{29.70}\std{.26}
& \textbf{29.83}\std{.09} \\

\bottomrule
\end{tabular}
\end{adjustbox}
\caption{Distilling KOALA-700M in the unfiltered setting. \textbf{RC}: Random Conditioning. \textbf{Static}: uniform sampling over clusters followed by uniform sampling within the selected cluster.}
\label{tab:policy}
\end{table}

One phenomenon we observe is that ARIA significantly accelerates convergence for models trained from random initialization (the C-based architectures).
For example, as shown in Fig.~\ref{fig:efficiency}, the C-Tiny student trained with ARIA reaches comparable FID and CLIP scores after approximately $175$K iterations, which corresponds to $60\%$ of the full training budget, compared to a model trained with RC for $300$K iterations.
Similar trends are observed across other C-based architectures; additional results are provided in Sec.~\ref{sec:supp_additional_eval} of the Supp. Mat.

Subfigures~(b)--(d) plot FID, CLIP, and IS across training checkpoints for the C-Tiny architecture.
ARIA achieves lower FID and higher CLIP earlier, demonstrating faster convergence than RC.
Importantly, these improvements are obtained with negligible computational overhead, as illustrated in Fig.~\ref{fig:efficiency}(a).

\subsection{Filtered Concept-Removal Setting}
\label{sec:exp_imbalanced}

Tables~\ref{tab:filtered}, \ref{tab:filtered_unseen_extra} report results in a more challenging regime
where samples related to \emph{animals, vehicles, and food} are removed
from the cached image--text pairs.
Students are therefore trained using $168$K cached pairs while retaining
the same auxiliary prompt pool $\mathcal{C}_{\text{aux}}$ of size $20$M.
For evaluation, we split the $30$K image--text pairs sampled from the MS-COCO validation set into two subsets:
(i) $20{,}117$ samples that do not contain \emph{animals, vehicles, or food} concepts (seen),
and (ii) $9{,}883$ samples that contain only these removed concepts (unseen). In this setting, the gap between ARIA and RC widens, particularly on the unseen subset corresponding to the removed semantic categories.
Intuitively, when the cached image--text pairs lack entire semantic categories, parts of the conditioning space become poorly covered by training data.
Prompts from such regions often produce larger disagreement between the student and the teacher, since the student has not observed sufficient examples of these concepts.
Since ARIA tracks the student--teacher discrepancy at the region level, these regions naturally accumulate higher discrepancy scores and are therefore sampled more frequently.

This behavior allows ARIA to focus training effort on parts of the conditioning space that remain under-distilled.
In contrast, uniform sampling allocates the same budget to both easy and difficult regions, which can lead to wasted updates on already well-aligned concepts.
As a result, ARIA yields noticeably larger improvements in this filtered setting than in the balanced regime, highlighting the benefit of adaptive allocation when the training distribution is imbalanced or systematically missing semantic categories.

\begin{table}[htb!]
\centering
\scriptsize
\setlength{\tabcolsep}{2.0pt}
\renewcommand{\arraystretch}{1.05}
\begin{adjustbox}{max width=\columnwidth}
\begin{tabular}{l ccc | ccc}
\toprule
\textbf{Model}
& \multicolumn{3}{c|}{RC}
& \multicolumn{3}{c}{Ours} \\

\cmidrule(lr){2-4} \cmidrule(lr){5-7}

& FID$\downarrow$ & IS$\uparrow$ & CLIP$\uparrow$
& FID$\downarrow$ & IS$\uparrow$ & CLIP$\uparrow$ \\

\midrule

\rowcolor{gray!10}
\textbf{SD 2.1}
& \multicolumn{6}{c}{\textbf{19.75\std{.11} / 24.85\std{.12} / 30.88\std{.05}}} \\

BK-v2-S
& 20.40\std{.21} & 23.63\std{.24} & 28.87\std{.02}
& \textbf{19.14}\std{.36} & \textbf{24.01}\std{.17} & \textbf{29.25}\std{.03} \\

BK-v2-T
& 20.84\std{.31} & 23.49\std{.19} & 28.82\std{.06}
& \textbf{19.67}\std{.39} & \textbf{23.52}\std{.24} & \textbf{29.07}\std{.02} \\

\midrule

\rowcolor{gray!10}
\textbf{SDXL}
& \multicolumn{6}{c}{\textbf{18.66\std{.17} / 26.21\std{.14} / 33.06\std{.02}}} \\

KOALA
& 26.42\std{.27} & 21.47\std{.37} & 29.83\std{.06}
& \textbf{24.89}\std{.31} & \textbf{22.54}\std{.31} & \textbf{30.18}\std{.08} \\

\bottomrule
\end{tabular}
\end{adjustbox}
\caption{Filtered concept-removal setting on SD~2.1 and SDXL; unseen-concept results only. ARIA improves over RC across backbones and student architectures.}
\label{tab:filtered_unseen_extra}
\end{table}

For the SD~2.1 and SDXL experiments in Table~\ref{tab:filtered_unseen_extra}, we report only unseen-concept results due to space constraints. In the seen regime, ARIA remains comparable to RC and achieves slight gains across most metrics.

\section{Conclusion and Future Work}
\label{sec:conclusion}

\textbf{Conclusion.}
In this work, we introduced \emph{ARIA}, a framework for adaptive importance allocation in large-scale training regimes. ARIA operates at the region level, dynamically reallocating training effort toward coarse regions that exhibit higher model discrepancy, without modifying the underlying optimization objective. We instantiate ARIA for text-to-image diffusion distillation, where large conditioning spaces make uniform sampling inefficient. We provide a theoretical analysis of the proposed scoring mechanism, establishing finite-time tracking guarantees for region-level discrepancy estimates.  Empirically, ARIA improves over uniform random sampling in most evaluated settings, with the strongest gains appearing under imbalanced or underrepresented conditioning regimes, and achieves faster convergence through efficient allocation.

\textbf{Future Work.}
Two directions emerge from this work. First, adaptive importance allocation is increasingly relevant in large-scale learning systems, where data pools are massive and unevenly informative. Extending ARIA to domains such as large language models and multimodal generative systems is a natural direction. Second, our region-based formulation opens the door to alternative allocation strategies. Beyond discrepancy-driven allocation, future work may explore uncertainty-aware or dynamically learned region constructions, potentially leading to more principled and robust training mechanisms.
\section{Limitations}
In our instantiation, ARIA constructs regions using fixed clustering over CLIP text embeddings. Although effective in our experiments, this choice may not capture fine-grained prompt differences, compositional structure, or task-specific notions of similarity. ARIA is not tied to CLIP or to fixed clusters, and future work could explore alternative encoders, task-specific or model-derived representations, soft or hierarchical regions, and dynamically updated partitions. Second, ARIA improves the allocation of a limited training budget, but it does not remove the cost of constructing the cached image set used for distillation. Following prior work, we assume a fixed cache of generated images and focus on selecting auxiliary text conditions more effectively. Reducing or eliminating this dependence on cached images remains an important direction. Finally, ARIA introduces practical design choices, including the embedding model, number of regions, EMA parameter, and score-to-probability mapping. Our ablations suggest that ARIA is stable, but different datasets or teacher–student pairs may benefit from different settings. More principled or learned allocation rules could ease deployment.

\newcommand{\acksection}{\section*{Acknowledgments and Disclosure of Funding}}
\NewEnviron{ack}{%
  \acksection
  \BODY
}
\begin{ack}

This research was funded by the Ministry of Science, Research and the Arts Baden-Wuerttemberg in the Artificial Intelligence Software Academy (AISA). We also acknowledge the support of the Stuttgart Center for Simulation Science (SimTech) and thank the International Max Planck Research School for Intelligent Systems (IMPRS-IS) for support. L.\ Mualem was supported by a postdoctoral scholarship from the Planning and Budgeting Committee (PBC) of the Council for Higher Education in Israel. The authors gratefully acknowledge the computing time provided on the high-performance computer HoreKa by the National High-Performance Computing Center at KIT (NHR@KIT). This center is jointly supported by the Federal Ministry of Education and Research and the Ministry of Science, Research and the Arts of Baden-Württemberg, as part of the National High-Performance Computing (NHR) joint funding program (https://www.nhr-verein.de/en/our-partners). HoreKa is partly funded by the German Research Foundation (DFG). 

\end{ack}

{
    \small
    \bibliographystyle{ieeenat_fullname}
    \bibliography{tex/main_eccv}
}
\newpage
\onecolumn
\appendix
\section{Additional Related Work}\label{app:related}

\paragraph{Relation to bandit-style allocation.}
ARIA can be interpreted through a multi-armed bandit lens~\cite{slivkins2019introduction,tang2025adapting,bouneffouf2020survey}, where each region corresponds to an arm and the observed discrepancy acts as a reward signal used to guide future sampling. However, unlike classical bandit settings, the reward distribution is non-stationary and depends on the evolving model parameters. Our analysis therefore focuses on the stability and tracking properties of the region-level estimator rather than cumulative regret. This perspective connects ARIA to adaptive data selection methods while highlighting the distinct challenges of training-time non-stationarity.

\paragraph{Coresets.}
Coresets~\cite{jubran2019introduction,feldman2009private,huang2019coresets,tukan2025improving} have been widely studied for large-scale machine learning and
clustering problems, in particular for objectives such as $k$-means and
$k$-median~\cite{har2004coresets,cohen2022improved,har2005smaller}. In machine learning, we are (usually) given an input set
$P \subseteq \mathbb{R}^d$ of $n$ points, its corresponding weight
function $w : P \rightarrow \mathbb{R}$, a feasible set of queries $X$,
and a loss function $\phi : P \times X \rightarrow [0,\infty)$.
The tuple $(P,w,X,\phi)$ is called a \emph{query space}, and it defines
the optimization problem at hand. For a given problem that is defined
by its query space $(P,w,X,\phi)$, and an error parameter
$\varepsilon \in (0,1)$, an $\varepsilon$-coreset is a small weighted
subset of the input points that approximates the loss of the input set
$P$ for every feasible query $x$, up to a provable bound of $1+\varepsilon$.
\section{Theorem Proof}\label{app:thm}
In this section, we present the full tracking guarantee theorem and provide its complete proof.
\paragraph{Why a theoretical guarantee is needed.}
ARIA maintains an exponential moving average (EMA) score for each semantic region to estimate its current training difficulty. These scores determine the sampling distribution used during training, and therefore inaccurate estimates could cause the algorithm to focus on the wrong regions. However, the discrepancy signal observed during training is stochastic, and the true region difficulty evolves as the student model improves. It is therefore important to understand whether the maintained EMA scores reliably track the underlying region difficulty.

\paragraph{Interpretation of the guarantee.}
Our analysis shows that the EMA score provably tracks the conditional mean discrepancy of each region. The resulting bound decomposes the tracking error into three components: an initialization term that decays geometrically, a drift term capturing how the true region difficulty evolves during training, and a stochastic term reflecting sampling noise.

In the bound, $\sigma$ measures the variance of the stochastic discrepancy observations arising from minibatch sampling and diffusion noise. The quantity $V$ measures the cumulative discounted drift of the true region difficulty, defined through the changes in $J_s^{(n)}$ over time. In realistic training dynamics, the discrepancy of a region evolves as the student model updates its parameters and progressively improves. However, stochastic optimization typically produces gradual parameter updates, and therefore the difficulty of a region tends to evolve smoothly across training iterations. The discounted drift assumption therefore allows the analysis to capture this non-stationary behavior while remaining mild in practice.

The theorem provides two complementary guarantees. The RMS bound controls the expected tracking error, while the high-probability bound shows that the tracking error remains uniformly bounded over an entire training horizon. Together, these results justify the use of EMA scores in ARIA as reliable estimates of region difficulty for guiding the sampling distribution during training.
\subsection{Preliminaries for the analysis}
\label{sec:analysis_prelim}

We analyze the behavior of the exponential moving average (EMA) score for a fixed region (cluster) $s$.

Since the score of $s$ is updated only when $s$ is selected, we consider the subsequence of training steps in which $s$ is chosen.
Let $n=1,2,\dots$ index the $n$-th update of region $s$ (not global training iterations).

Let $\Delta_s^{(n)}$ denote the scalar discrepancy observed at the $n$-th update of $s$.
Let $\{\mathcal{F}_{n-1}\}_{n\ge1}$ be the filtration representing all information available just before that update,
including previously sampled data, past discrepancies, and the current model parameters.

Define the conditional mean discrepancy
\[
J_s^{(n)} := \mathbb{E}\!\left[\Delta_s^{(n)} \mid \mathcal{F}_{n-1}\right],
\]
and the noise term
\[
\xi_n := \Delta_s^{(n)} - J_s^{(n)}.
\]

and we define the tracking error
\[
e_n := D_s^{(n)} - J_s^{(n)}.
\]

\setcounter{theorem}{1}
\begin{restatable}{theorem}{AriaTheorem}
\label{thm:ema_tracking_combined}
Fix a cluster $s$ and consider the subsequence of training iterations in which $s$ is selected,
indexed by $n=1,2,\dots$.
Let $\{\mathcal{F}_{n-1}\}_{n\ge 1}$, $\Delta_s^{(n)}$, $J_s^{(n)}$, $\xi_n$, $D_s^{(n)}$, and
$e_n := D_s^{(n)}-J_s^{(n)}$ be as defined in the preliminaries, and let $\alpha\in(0,1)$.

Assume that for all $n\ge 1$:
\begin{enumerate}[label=(A\arabic*),leftmargin=*]
    \item \textbf{Unbiased noise:} $\mathbb{E}[\xi_n \mid \mathcal{F}_{n-1}] = 0$.
    \item \textbf{Bounded conditional variance:} $\mathbb{E}[\xi_n^2 \mid \mathcal{F}_{n-1}] \le \sigma^2$ for some $\sigma>0$.
    \item \textbf{Bounded discounted drift:} for some $V\ge 0$ and every $n\ge 1$,
    \begin{equation}
    \label{eq:disc_drift_thm_combined}
    V_s^{(\alpha)}(n)
    :=\sum_{k=1}^{n}(1-\alpha)^{\,n-k}\,\big|J_s^{(k)}-J_s^{(k-1)}\big|
    \le V,
    \end{equation}
    where $J_s^{(0)}$ is any fixed reference value.
\end{enumerate}
Then the following statements hold.
\begin{enumerate}[label=(\roman*),leftmargin=*]
\item \textbf{(RMS tracking bound).} For every $n\ge 1$,
\begin{equation}
\label{eq:ema_tracking_bound_combined}
\sqrt{\mathbb{E}[e_n^2]}
\;\le\;
(1-\alpha)^n |e_0|
\;+\;
(1-\alpha)\,V
\;+\;
\sigma \sqrt{\alpha}.
\end{equation}

\item \textbf{(High-probability tracking bound, uniform up to a horizon).}
Assume in addition that
\begin{enumerate}[label=(A\arabic*),leftmargin=*]
\setcounter{enumi}{3}
\item \textbf{Bounded noise:} $|\xi_n|\le b$ almost surely for some $b>0$.
\end{enumerate}
Fix any horizon $N\ge 1$ and any $\delta\in(0,1)$. Then with probability at least $1-\delta$,
simultaneously for all $n\in\{1,\dots,N\}$,
\begin{equation}
\label{eq:ema_hp_bound_uniform}
|e_n|
\;\le\;
(1-\alpha)^n|e_0|
\;+\;
(1-\alpha)V
\;+\;
\sqrt{2\,\alpha\,\sigma^2\log\!\frac{2N}{\delta}}
\;+\;
\frac{2}{3}\,\alpha\,b\,\log\!\frac{2N}{\delta}.
\end{equation}
\end{enumerate}
\end{restatable}

\begin{proof}
\textbf{(i) RMS bound.}
Let $e_n := D_s^{(n)} - J_s^{(n)}$ denote the tracking error. Starting from the EMA update for cluster $s$ and using the decomposition $\Delta_s^{(n)}=J_s^{(n)}+\xi_n$, we can write
\[
D_s^{(n)}=(1-\alpha)D_s^{(n-1)}+\alpha\bigl(J_s^{(n)}+\xi_n\bigr).
\]
Subtracting $J_s^{(n)}$ and adding/subtracting $(1-\alpha)J_s^{(n-1)}$ yields
\begin{equation}
\label{eq:error_rec_combined}
 e_n=(1-\alpha)e_{n-1}+(1-\alpha)\bigl(J_s^{(n-1)}-J_s^{(n)}\bigr)+\alpha\xi_n.
\end{equation}

Define the drift increment $\delta_n:=J_s^{(n-1)}-J_s^{(n)}$. Repeated substitution of \eqref{eq:error_rec_combined} yields
\begin{equation}
\label{eq:decomp_combined}
e_n
=
(1-\alpha)^n e_0
+
(1-\alpha)\sum_{k=1}^{n}(1-\alpha)^{\,n-k}\,\delta_k
+
\alpha\sum_{k=1}^{n}(1-\alpha)^{\,n-k}\,\xi_k.
\end{equation}
For convenience, denote the three terms in \eqref{eq:decomp_combined} by
\[
A_n:=(1-\alpha)^n e_0,\qquad
B_n:=(1-\alpha)\sum_{k=1}^{n}(1-\alpha)^{\,n-k}\,\delta_k,\qquad
C_n:=\alpha\sum_{k=1}^{n}(1-\alpha)^{\,n-k}\,\xi_k,
\]
so that $e_n=A_n+B_n+C_n$.

The initialization term satisfies
\[
|A_n|=(1-\alpha)^n|e_0|.
\]
For the drift term, apply the triangle inequality and $\delta_k=J_s^{(k-1)}-J_s^{(k)}$:
\[
|B_n|
\le
(1-\alpha)\sum_{k=1}^{n}(1-\alpha)^{\,n-k}|\delta_k|
=
(1-\alpha)\sum_{k=1}^{n}(1-\alpha)^{\,n-k}\bigl|J_s^{(k)}-J_s^{(k-1)}\bigr|.
\]
By assumption (A3) this discounted drift sum is at most $V$, hence
\begin{equation}
\label{eq:Bn_bound_combined}
|B_n|\le (1-\alpha)V.
\end{equation}

We next control the stochastic term $C_n$ in mean square. From the definition of $C_n$,
\[
C_n=\alpha\sum_{k=1}^{n} w_{k,n}\,\xi_k,\qquad
w_{k,n}:=(1-\alpha)^{n-k}.
\]
Expanding $C_n^2$ and taking expectations gives
\[
\mathbb{E}[C_n^2]
=
\alpha^2\,
\mathbb{E}\!\left[
\sum_{k=1}^{n}\sum_{j=1}^{n}
w_{k,n}w_{j,n}\,\xi_k\xi_j
\right].
\]
We show that cross terms vanish: for $j<k$,
$\xi_j$ is $\mathcal{F}_{k-1}$-measurable, hence
\[
\mathbb{E}[\xi_k\xi_j]
=
\mathbb{E}\!\left[\mathbb{E}[\xi_k\xi_j\mid \mathcal{F}_{k-1}]\right]
=
\mathbb{E}\!\left[\xi_j\,\mathbb{E}[\xi_k\mid \mathcal{F}_{k-1}]\right]
=0,
\]
where the last equality uses (A1).
Therefore only diagonal terms remain:
\[
\mathbb{E}[C_n^2]
=
\alpha^2\sum_{k=1}^{n} w_{k,n}^2\,\mathbb{E}[\xi_k^2].
\]
Using (A2) and iterated expectation,
\[
\mathbb{E}[\xi_k^2]
=
\mathbb{E}\!\left[\mathbb{E}[\xi_k^2\mid \mathcal{F}_{k-1}]\right]
\le
\sigma^2,
\]
so
\[
\mathbb{E}[C_n^2]
\le
\alpha^2\sigma^2
\sum_{k=1}^{n} (1-\alpha)^{2(n-k)}.
\]
The remaining sum is geometric:
\[
\sum_{k=1}^{n} (1-\alpha)^{2(n-k)}
=
\sum_{j=0}^{n-1}(1-\alpha)^{2j}
\le
\frac{1}{1-(1-\alpha)^2}
=
\frac{1}{2\alpha-\alpha^2}
\le
\frac{1}{\alpha}.
\]
Hence
\begin{equation}
\label{eq:Cn_second_moment_combined}
\mathbb{E}[C_n^2]\le \alpha\sigma^2
\qquad\Rightarrow\qquad
\sqrt{\mathbb{E}[C_n^2]}\le \sigma\sqrt{\alpha}.
\end{equation}

Combining the above bounds, from $e_n=A_n+B_n+C_n$ and the triangle inequality,
\[
|e_n|\le |A_n|+|B_n|+|C_n|.
\]
Taking $L_2$-norms (i.e., $\|X\|_2:=\sqrt{\mathbb{E}[X^2]}$) and using Minkowski's inequality,
\[
\sqrt{\mathbb{E}[e_n^2]}
\le
\sqrt{\mathbb{E}[A_n^2]}+\sqrt{\mathbb{E}[B_n^2]}+\sqrt{\mathbb{E}[C_n^2]}.
\]
Since $A_n$ and $B_n$ are deterministic given $e_0$ and $\{J_s^{(k)}\}$,
$\sqrt{\mathbb{E}[A_n^2]}=|A_n|$ and $\sqrt{\mathbb{E}[B_n^2]}=|B_n|$.
Applying the deterministic bounds above and \eqref{eq:Cn_second_moment_combined} gives
\[
\sqrt{\mathbb{E}[e_n^2]}
\le
(1-\alpha)^n|e_0|+(1-\alpha)V+\sigma\sqrt{\alpha},
\]
which is exactly \eqref{eq:ema_tracking_bound_combined}.

\paragraph{(ii) Uniform high-probability bound.}
We now prove the uniform high-probability bound. Fix an $n\in\{1,\dots,N\}$. By the decomposition $e_n=A_n+B_n+C_n$ and the deterministic bounds above,
\[
|e_n|
\le
|A_n|+|B_n|+|C_n|
\le
(1-\alpha)^n|e_0|+(1-\alpha)V+|C_n|.
\]
Thus it suffices to upper bound $|C_n|$ with high probability.

Recall $C_n=\sum_{k=1}^{n} X_k$ where we define
\[
X_k := \alpha\,w_{k,n}\,\xi_k,
\qquad
w_{k,n}:=(1-\alpha)^{n-k}.
\]
We claim that $(X_k)_{k=1}^{n}$ is a martingale difference sequence with respect to $(\mathcal{F}_k)_{k=0}^{n}$.
Indeed, since $w_{k,n}$ is deterministic and $\xi_k$ is $\mathcal{F}_k$-measurable, $X_k$ is $\mathcal{F}_k$-measurable.
Moreover, by (A1),
\[
\mathbb{E}[X_k\mid \mathcal{F}_{k-1}]
=
\alpha w_{k,n}\,\mathbb{E}[\xi_k\mid \mathcal{F}_{k-1}]
=
0.
\]
By (A4) and $w_{k,n}\le 1$,
\[
|X_k|
=
\alpha w_{k,n}|\xi_k|
\le
\alpha b.
\]
Finally, by (A2),
\[
\mathbb{E}[X_k^2\mid \mathcal{F}_{k-1}]
=
\alpha^2 w_{k,n}^2\,\mathbb{E}[\xi_k^2\mid \mathcal{F}_{k-1}]
\le
\alpha^2 w_{k,n}^2\,\sigma^2.
\]
Summing over $k=1,\dots,n$,
\[
\sum_{k=1}^{n}\mathbb{E}[X_k^2\mid \mathcal{F}_{k-1}]
\le
\alpha^2\sigma^2\sum_{k=1}^{n}w_{k,n}^2
=
\alpha^2\sigma^2\sum_{k=1}^{n}(1-\alpha)^{2(n-k)}
\le
\alpha^2\sigma^2\cdot\frac{1}{\alpha}
=
\alpha\sigma^2,
\]
where we used the same geometric-sum bound as above.
Thus the predictable quadratic variation is at most $v:=\alpha\sigma^2$.

Freedman's inequality (for martingale differences with bounded increments) states that if
$\sum_{k=1}^{n}\mathbb{E}[X_k^2\mid \mathcal{F}_{k-1}]\le v$ with high probability and $|X_k|\le c$ with high probability,
then for any $\eta\in(0,1)$,
\[
\Pr\!\left(\sum_{k=1}^{n}X_k \ge \sqrt{2v\log(1/\eta)}+\frac{2}{3}c\log(1/\eta)\right)\le \eta.
\]
We apply this with $c=\alpha b$, $v=\alpha\sigma^2$, and $\sum_{k=1}^{n}X_k=C_n$. This gives
\[
\Pr\!\left(C_n \ge \sqrt{2\alpha\sigma^2\log(1/\eta)}+\frac{2}{3}\alpha b\log(1/\eta)\right)\le \eta.
\]
Applying the same bound to $-C_n$ yields
\[
\Pr\!\left(|C_n| \ge \sqrt{2\alpha\sigma^2\log(2/\eta)}+\frac{2}{3}\alpha b\log(2/\eta)\right)\le \eta,
\]
where we used a union bound for the two tails.

Set $\eta:=\delta/N$ and apply the above inequality to each $n\in\{1,\dots,N\}$.
Then for each fixed $n$,
\[
\Pr\!\left(|C_n| \ge \sqrt{2\alpha\sigma^2\log\!\frac{2N}{\delta}}+\frac{2}{3}\alpha b\log\!\frac{2N}{\delta}\right)\le \frac{\delta}{N}.
\]
Taking a union bound over $n=1,\dots,N$ shows that with probability at least $1-\delta$,
simultaneously for all $n\in\{1,\dots,N\}$,
\begin{equation}
\label{eq:Cn_uniform_bound}
|C_n|
\le
\sqrt{2\alpha\sigma^2\log\!\frac{2N}{\delta}}
+\frac{2}{3}\alpha b\log\!\frac{2N}{\delta}.
\end{equation}

On the same event, for every $n\le N$,
\[
|e_n|
\le
(1-\alpha)^n|e_0|+(1-\alpha)V+|C_n|.
\]
Substituting \eqref{eq:Cn_uniform_bound} gives exactly \eqref{eq:ema_hp_bound_uniform}.
This completes the proof.
\end{proof}
\section{Additional Experiments}\label{supp:additional_exp}
\subsection{Implementation Details}

We adopt the same experimental setup used in~\cite{kim2025random}. All models are trained using the AdamW optimizer with a learning rate of $5\times10^{-5}$. Training is performed on four NVIDIA A100 GPUs with 80GB memory using a batch size of 256. 

We employ the same loss functions described in Eq.~(1) and Eq.~(2) in the main paper, with equal weights of 1. The feature-level distillation loss is applied after each block of the U-Net architecture.

\subsection{Embedding Similarity and Training Signal Consistency}
\label{sec:appendix_embedding_signal}

ARIA clusters the conditioning space using text embeddings.
To validate that this geometry is meaningful for optimization, we study
whether prompts that are close in embedding space also induce similar
distillation signals.

For a noisy latent $x_t$ and conditioning $c$, the distillation residual is
\[
r(x_t,c)=\hat{\epsilon}_S(x_t,c)-\hat{\epsilon}_T(x_t,c),
\]
where $\hat{\epsilon}_S$ and $\hat{\epsilon}_T$ are the student and teacher
denoising predictions.
Since this residual determines the update applied to the student, we
analyze both the residual similarity itself and the similarity between
the corresponding gradient directions.

Figure~\ref{fig:embedding_training_signal} shows the cosine distance
between training signals as a function of cosine similarity in text
embedding space.
The left subfigure reports denoising residual distance and the right
subfigure reports gradient direction distance, each at two noise levels.
In all cases, higher embedding similarity corresponds to smaller
training-signal distance.
This indicates that nearby prompts in the conditioning space tend to
produce similar optimization signals, supporting the use of
embedding-based clustering for defining regions in ARIA.

\begin{figure*}[htb!]
\centering
\begin{subfigure}{0.49\linewidth}
    \centering
    \includegraphics[width=\linewidth]{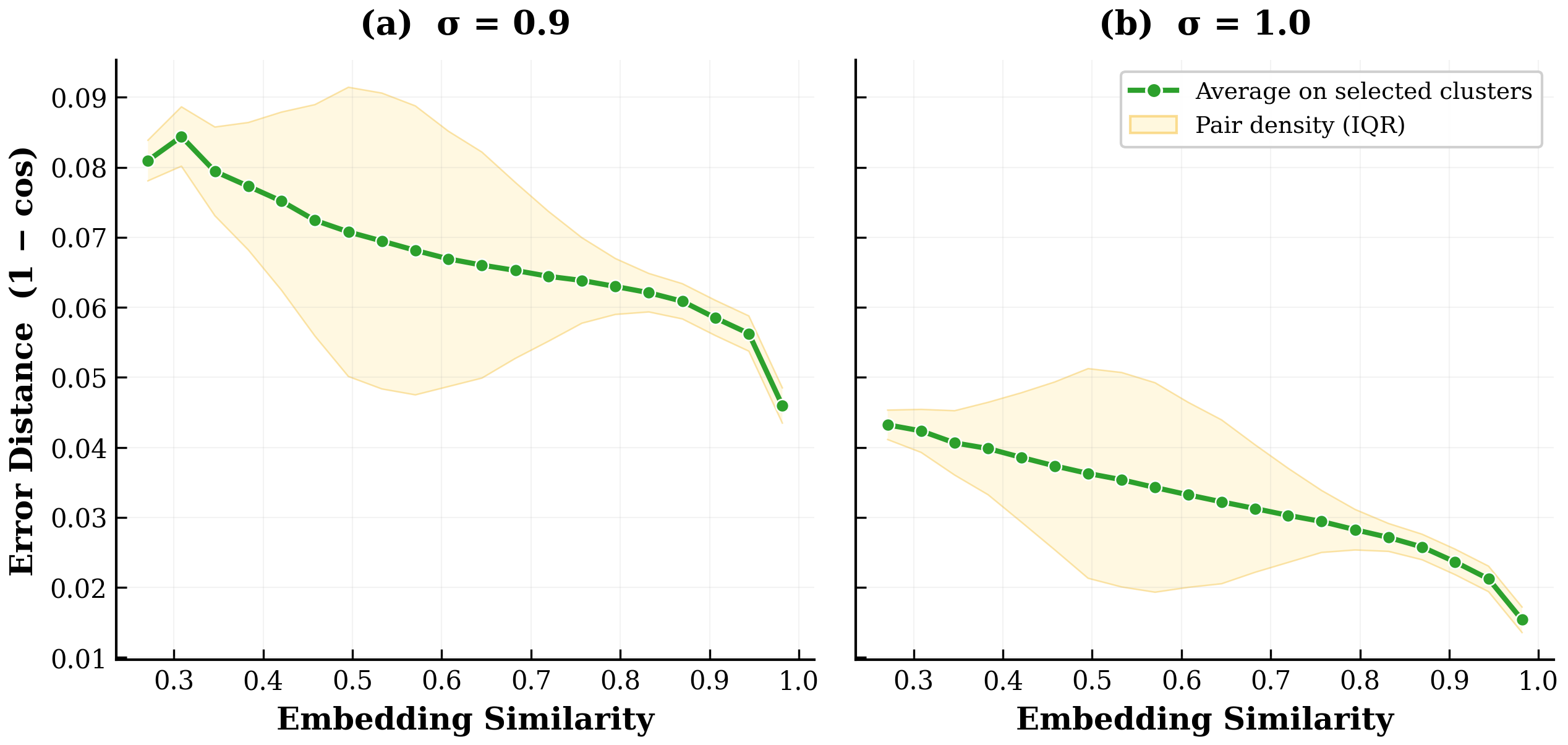}
    \caption{Denoising residual distance versus embedding similarity.}
    \label{fig:embed_noise_corr}
\end{subfigure}
\hfill
\begin{subfigure}{0.49\linewidth}
    \centering
    \includegraphics[width=\linewidth]{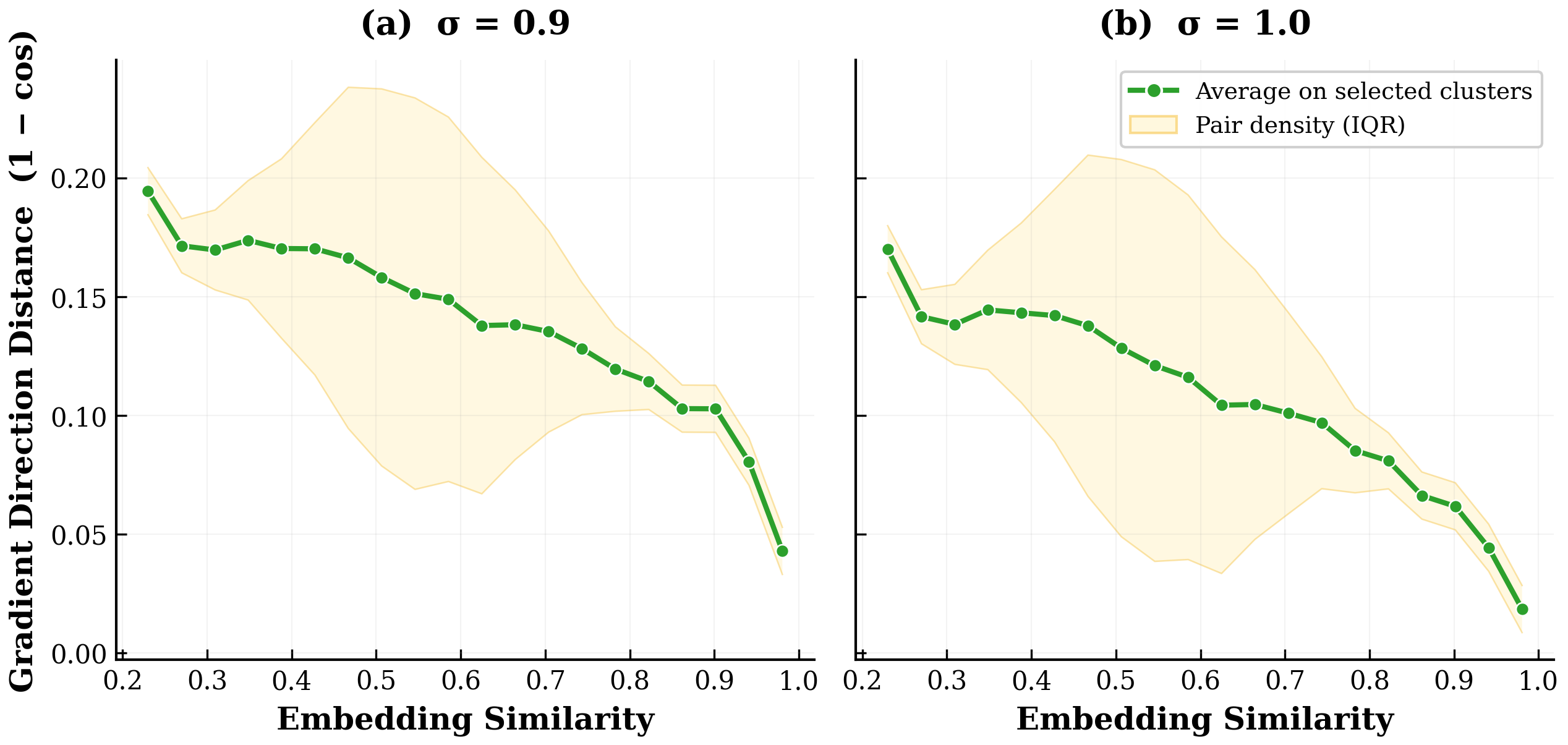}
    \caption{Gradient direction distance versus embedding similarity.}
    \label{fig:embed_grad_corr}
\end{subfigure}
\caption{
Relationship between text embedding similarity and the similarity of the
distillation training signal.
Each subfigure already contains two panels corresponding to $\sigma=0.9$
and $\sigma=1.0$.
The left subfigure shows denoising residual distance, while the right
subfigure shows gradient direction distance.
Across both noise levels and both signal types, prompts that are closer
in embedding space consistently induce more similar training signals.
The green curves show the average over selected clusters, and the shaded
regions indicate the interquartile range of pairwise distances.
}
\label{fig:embedding_training_signal}
\end{figure*}

\subsection{Sampling Distribution Analysis}
\label{sec:temperature_sampling}

After clustering the text corpus into semantic regions, ARIA maintains a
difficulty score $D_s$ for each cluster using the EMA statistics described in
Sec.~2. The remaining design choice is how these
difficulty scores are converted into a sampling distribution over clusters.

To control the sharpness of the distribution, we first apply a power
transformation to the difficulty scores
\begin{equation}
    w_s = D_s^\beta,
\end{equation}
with $\beta=0.6$ in our experiments. This transformation amplifies relative
differences between clusters and produces a smoother long-tailed distribution
over difficulty levels.

We then convert the transformed scores into sampling probabilities using one of
two normalization schemes. The first is a power-law allocation
\begin{equation}
    p_s = \frac{w_s}{\sum_{j=1}^{S} w_j},
\end{equation}
which directly allocates probability proportionally to the transformed
difficulty. The second uses a temperature-scaled softmax
\begin{equation}
    p_s = \frac{\exp(w_s / T)}{\sum_{j=1}^{S} \exp(w_j / T)},
\end{equation}
where the temperature $T$ controls the concentration of probability mass.
Large values of $T$ produce distributions close to uniform sampling, while
smaller values increasingly concentrate probability on the hardest clusters.

Figure~\ref{fig:temperature} illustrates the resulting sampling behavior.
Figure~\ref{fig:temperature}(a) shows the probability distributions over
clusters ranked by difficulty. Uniform sampling allocates equal probability to
all clusters, whereas decreasing the softmax temperature or using power-law
weighting increasingly concentrates probability on the hardest clusters.

However, overly aggressive focusing can reduce coverage of easier regions.
Fig.~\ref{fig:temperature}(b) shows that sharper distributions (e.g.,
$T{=}0.05$ or power-law weighting) allocate a larger fraction of samples to the
top 10\% hardest clusters while reducing the share assigned to the easiest
clusters. In our experiments, power-law weighting improved performance for
teacher-initialized models, but only matched or sometimes degraded performance
for randomly initialized students.

Finally, Fig.~\ref{fig:temperature}(c) reports the Spearman correlation between
cluster difficulty and sampling frequency over training. Lower temperatures and
power-law sampling produce stronger alignment between the sampling probability
and the estimated cluster difficulty.

Based on this analysis, we adopt $T{=}0.1$, which provides a moderate focusing
effect while preserving sufficient sampling coverage across all clusters.

\begin{figure}[t]
    \centering
    \includegraphics[width=\linewidth]{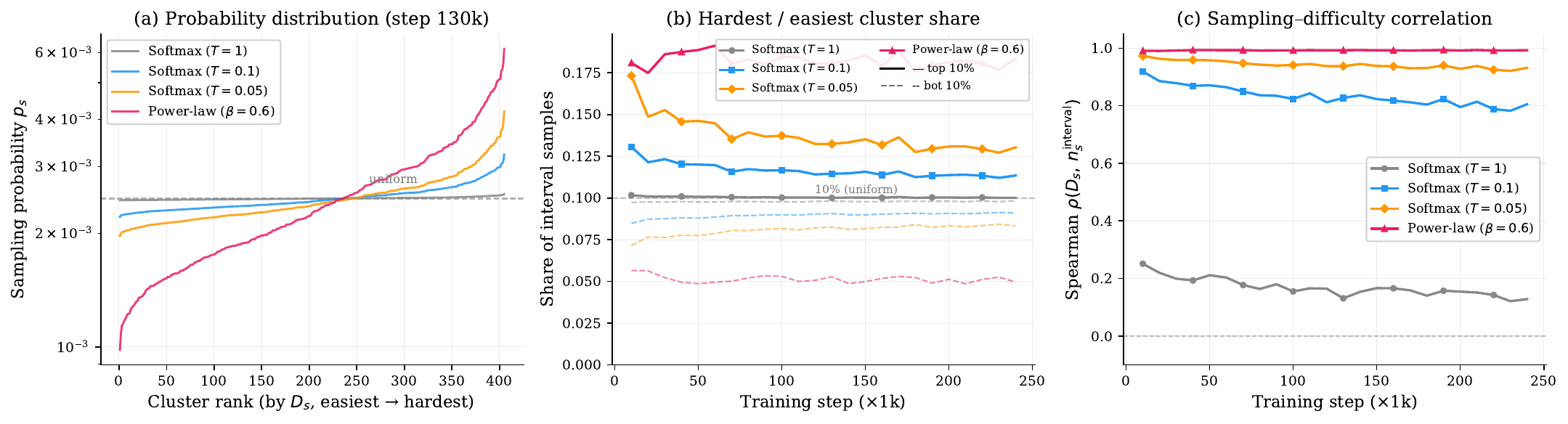}
    \caption{
    Analysis of sampling distributions over semantic clusters.
    (a) Probability distribution over clusters ranked by difficulty.
    Lower temperatures concentrate probability mass on harder clusters,
    while power-law weighting produces the sharpest allocation.
    (b) Fraction of samples allocated to the hardest and easiest 10\% of clusters during training.
    Aggressive focusing increases coverage of difficult regions but reduces sampling of easy clusters.
    (c) Spearman correlation between cluster difficulty $D_s$ and sampling frequency.
    Lower temperatures and power-law sampling produce stronger alignment with cluster difficulty.
    }
    \label{fig:temperature}
\end{figure}
\subsection{Additional Empirical Evaluation}
\label{sec:supp_additional_eval}

In what follows, we present additional empirical evaluations that further
illustrate the training dynamics of ARIA compared to Random Conditioning
(RC). Figure~\ref{fig:training_curves} reports the evolution of standard
generation metrics across training checkpoints for the C-Micro student
architecture.

Specifically, we track the Fréchet Inception Distance (FID), CLIP score,
and Inception Score (IS) as training progresses from $25$K to $300$K
iterations. These curves provide a more detailed view of the convergence
behavior discussed in the main paper.

As shown in Fig.~\ref{fig:training_curves}, ARIA consistently achieves
better performance earlier during training. In particular, ARIA reaches
lower FID and higher CLIP scores at substantially earlier checkpoints,
indicating faster alignment between the student and teacher models.
A similar trend is observed for the Inception Score, where ARIA maintains
higher values throughout most of the training trajectory.

These results further support the observation that ARIA improves the
efficiency of the distillation process by allocating training updates to
more informative regions of the conditioning space. As a result, the
student model converges faster while incurring negligible computational
overhead compared to RC.

\begin{figure}[t]
\centering
\includegraphics[width=\linewidth]{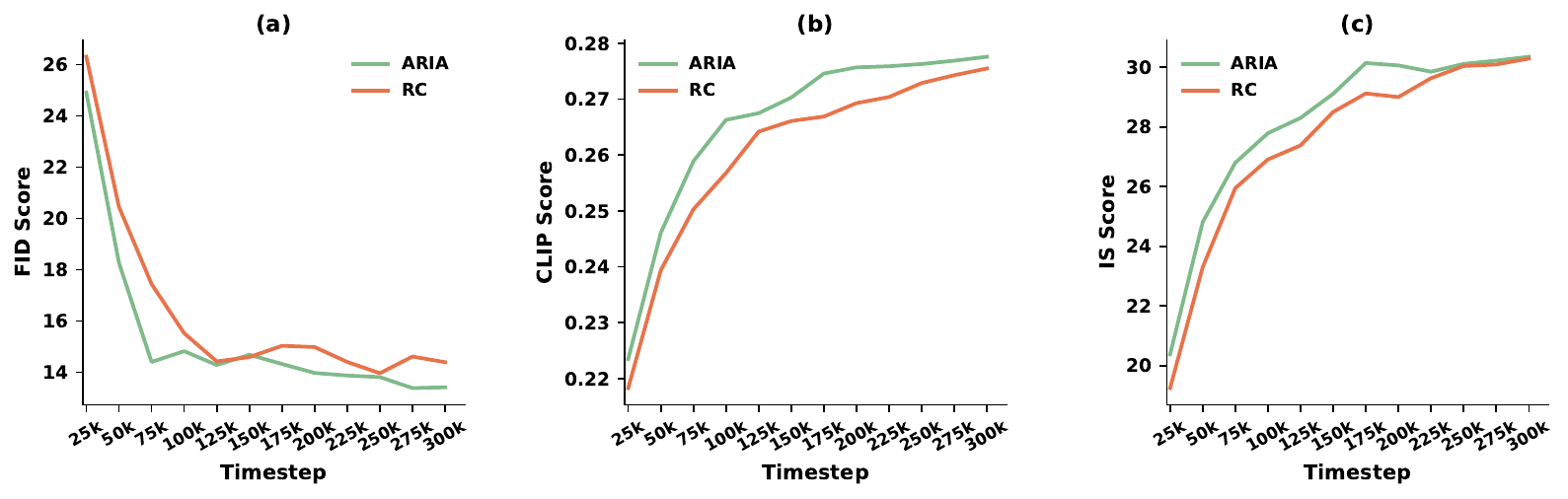}
\caption{
Training dynamics for the C-Micro student architecture.
We report (a) FID, (b) CLIP score, and (c) Inception Score (IS) across
training checkpoints from $25$K to $300$K iterations.
ARIA consistently achieves lower FID and higher CLIP scores earlier in
training compared to Random Conditioning (RC), demonstrating faster
convergence of the distilled student model.
}
\label{fig:training_curves}
\end{figure}

\begin{tcolorbox}[colback=blue!3!white,colframe=blue!60!black,title=Key Takeaway,width=\linewidth]
\textbf{ARIA significantly improves training efficiency by accelerating convergence of randomly initialized student models.}
Across training checkpoints, ARIA consistently achieves lower FID and higher CLIP scores earlier than RC,
indicating that the student reaches strong performance with substantially fewer optimization steps.

This property is particularly important in practical distillation settings where the student architecture differs from the teacher,
for example when training compressed or lightweight models that cannot reuse the teacher's weights.
In such scenarios the student must be trained from random initialization, making convergence speed a critical factor.

By directing updates toward harder regions of the conditioning space, ARIA allocates the training budget more effectively,
allowing high-quality students to be obtained with reduced compute. This is especially valuable in resource-constrained
environments where training large diffusion models is expensive and iterative experimentation (e.g., architecture search or
hyperparameter tuning) is required.
\end{tcolorbox}

\begin{figure}[p]
    \centering
    \includegraphics[width=\textwidth,height=0.8\textheight,keepaspectratio]{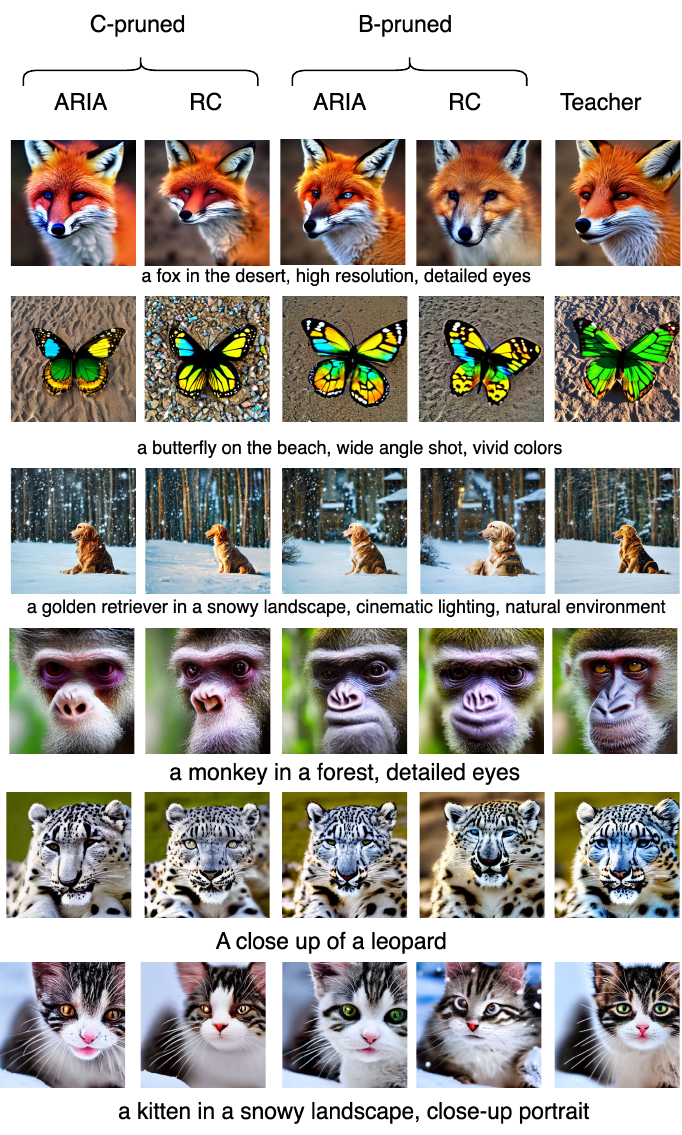}
    \caption{Qualitative comparison on animal-related prompts on models that were trained on animal-free images. From left to right, we show outputs from the channel-pruned student trained with ARIA, the same channel-pruned student trained with RC, the block-pruned student trained with ARIA, the same block-pruned student trained with RC, and the teacher.}
    \label{fig:appendix_animals_aria_vs_rc}
\end{figure}

\begin{figure}[p]
    \centering
    \includegraphics[width=\textwidth,height=0.8\textheight,keepaspectratio]{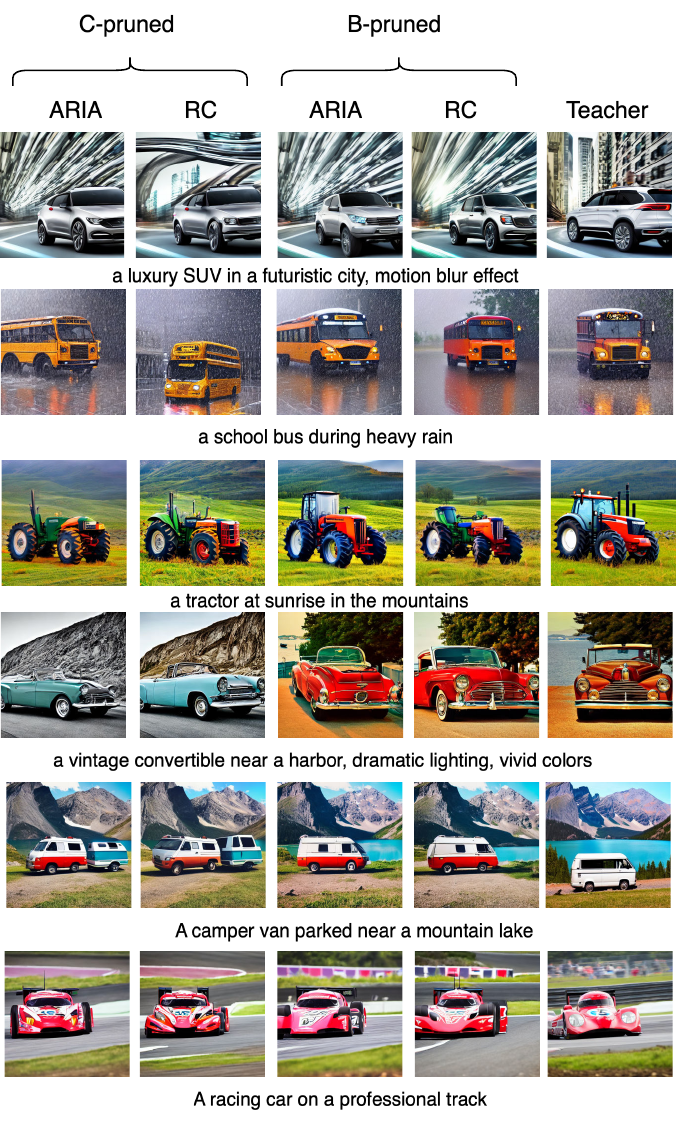}
    \caption{Qualitative comparison on vehicles-related prompts on models that were trained on images without vehicles.}
    \label{fig:appendix_vehicles_aria_vs_rc}
\end{figure}
\begin{figure}[p]
    \centering
    \includegraphics[width=\textwidth,height=0.92\textheight,keepaspectratio]{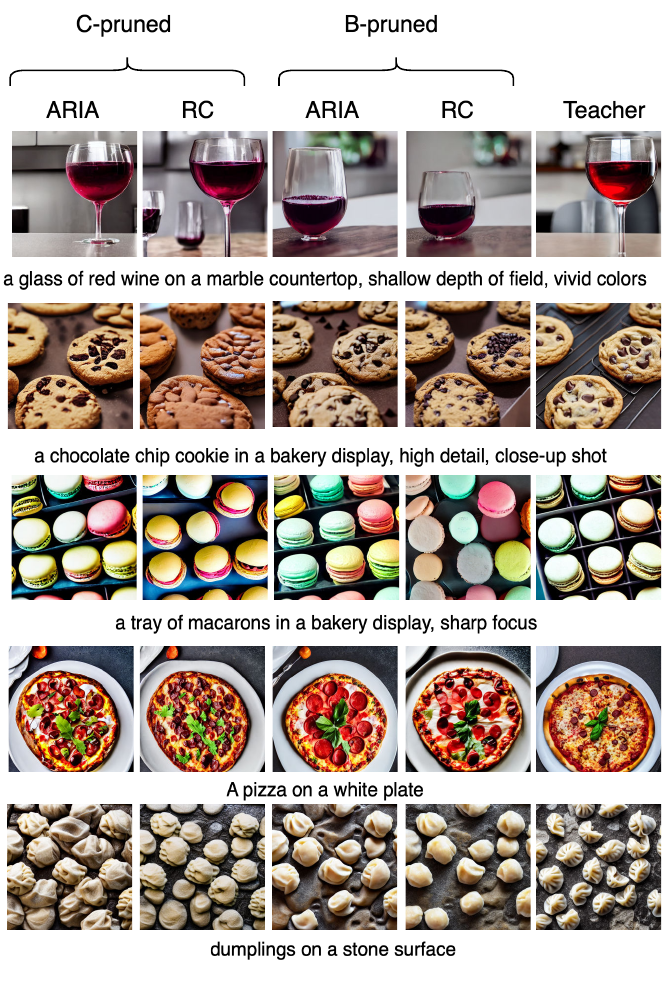}
    \caption{Qualitative comparison on food-related prompts on models that were trained on images without food. }
    \label{fig:appendix_food_aria_vs_rc}
\end{figure}
\begin{figure}[p]
    \centering
    \includegraphics[width=\textwidth,height=0.92\textheight,keepaspectratio]{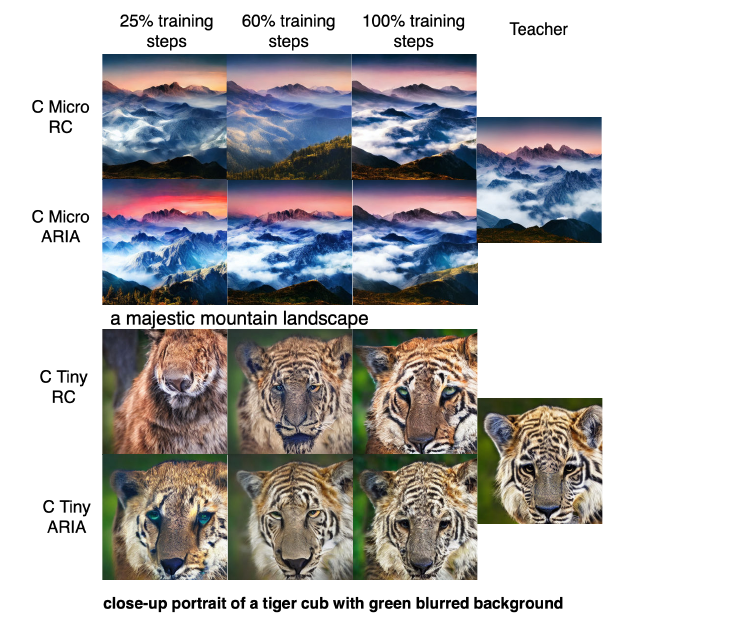}
    \caption{In our experiments, we observed that training with ARIA significantly improves convergence speed. This figure provides a qualitative comparison across training checkpoints, illustrating that ARIA converges faster than RC. We compare generations at 25\%, 60\%, and 100\% of training for two students trained from random initialization: C-Micro and C-Tiny.}
    \label{fig:appendix_checkpoints_aria_faster_convergence}
\end{figure}
\section{Clustering the data}\label{sec:clustering}
\subsection{Constructing Semantic Regions via CLIP Clustering}
\label{sec:region_construction}

Recall that our framework is agnostic to the specific procedure used to divide the training samples space into regions. In principle, any method that partitions the training samples space into  meaningful groups can be used. In our implementation, we instantiate regions by clustering text prompts in the embedding space of a pretrained CLIP text encoder. The main idea is that prompts that are close in CLIP space tend to share semantic content, and therefore clustering this space provides a practical way to define regions over which our cluster-level discrepancy statistics can be maintained. 

More concretely, let $\mathcal{C}=\{c_1,\dots,c_n\}$ denote the set of text conditions used to define the conditioning pool, and let $e_i \in \mathbb{R}^d$ be the CLIP embedding of condition $c_i$. We partition $\{e_i\}_{i=1}^n$ into $K$ clusters using $k$-means. Since the quality of the partition depends on the choice of $K$, we do not fix it a priori. Instead, we select $K$ based on the silhouette score~\cite{shahapure2020cluster}, which is a widely used criterion for evaluating the quality of a clustering by balancing within-cluster cohesion against between-cluster separation. Intuitively, a good value of $K$ should produce clusters whose members are close to one another while remaining well separated from points assigned to other clusters.  

A direct sweep over many candidate values of $K$ on the full embedding set is very expensive. Running $k$-means repeatedly on a very large prompt collection or training samples is already costly, and computing clustering-quality diagnostics on top of that adds further overhead. To make this step practical, we use a $k$-means coreset. A coreset is a weighted subset of points that approximately preserves the $k$-means objective of the original dataset, and therefore can be used as a much smaller proxy for downstream clustering computations. In our setting, this allows us to evaluate many candidate values of $K$ using only a compact weighted representative subset, instead of repeatedly clustering the full prompt pool. 

We want to use the sensitivity sampling framework  to compute a coreset for a set of points $P$ in $\mathcal{R}^d$. Sensitivity-based coreset constructions typically rely on an initial
approximate clustering solution in order to bound the importance
(sensitivity) of each point with respect to the $k$-means objective.
This is commonly obtained through a bicriteria solution, which provides
a constant-factor approximation to the optimal $k$-means cost while
possibly using more than $k$ centers. Such a solution allows one to
derive upper bounds on point sensitivities that can be used for
importance sampling.

\paragraph{Bicriteria initialization via $k$-means++.}
To build the coreset, we first obtain a bicriteria solution using $k$-means++ seeding. $k$-means++ chooses the first center uniformly at random and then samples each next center with probability proportional to the squared distance to the nearest previously chosen center. Arthur and Vassilvitskii~\cite{arthur2007k} showed that this $D^2$-sampling procedure yields an $O(\log K)$ approximation in expectation to the optimal $k$-means objective; specifically, their Theorem~3.1 proves
\[
\mathbb{E}[\phi] \le 8(\ln K + 2)\phi_{\mathrm{OPT}}.
\]
They also note that the running time of the seeding procedure is $O(nKd)$ for $n$ points in $\mathbb{R}^d$. In the language of coreset construction, this provides an $(\alpha,\beta)$ bicriteria solution with $\alpha = O(\log K)$ and $\beta=1$.  

\paragraph{Sensitivity-based coreset construction.}
Given this bicriteria solution, we follow the sensitivity-sampling framework of Braverman et al.~\cite{braverman2021efficient}. In particular, plugging the $(\alpha,\beta)$ assignment induced by $k$-means++ into their coreset construction yields an efficiently computable upper bound on the sensitivity of each point. Their framework samples points proportionally to such upper bounds, and Algorithm~1 constructs a weighted coreset whose size depends on the total sensitivity rather than directly on the original dataset size.  

In our implementation, after obtaining the $k$-means++ centers, each point $p$ is assigned to its closest center, inducing a cluster $\cluster(p)$ with center $\mu(\cluster(p))$. We then use the following sensitivity upper bound:
\[
s(p)
=
\frac{4\bigl(8(\log K + 2)+1\bigr)}{|\cluster(p)|}
+
\frac{16(\log K + 2)\,\|p-\mu(\cluster(p))\|_2^2}
{\sum_{q \in \cluster(p)} \|q-\mu(\cluster(p))\|_2^2 }.
\]
This is exactly the form obtained by combining the $k$-means++ bicriteria guarantee with the sensitivity framework of Braverman et al.~\cite{braverman2021efficient}. The first term increases the importance of small clusters, preventing them from being ignored, while the second term increases the importance of points that lie far from their assigned cluster center, which are precisely the points that contribute more to the $k$-means cost. Sampling according to these sensitivities therefore yields a weighted subset that better preserves the geometry relevant to clustering than uniform subsampling.  

Braverman et al.~\cite{braverman2021efficient} further show that, once such sensitivity upper bounds are available, one can construct an $\varepsilon$-coreset of size
\[
O\!\left(\frac{t}{\varepsilon^2}\bigl(d\log t + \log(1/\delta)\bigr)\right),
\]
where $t$ is the total sensitivity upper bound. For the $k$-means setting induced by the bicriteria above, this gives a coreset size scaling on the order of
\[
O\!\left(\frac{K d \log K}{\varepsilon^2}\right)
\]
up to logarithmic factors, and Theorem~34 in their supplement guarantees that the resulting weighted sample is an $\varepsilon$-coreset with high probability. Importantly for our use case, that theorem applies simultaneously to every $k' < K$ once the sensitivities are computed using a fixed upper bound $K$. 

\paragraph{Selecting the number of clusters.}
After constructing the coreset, we sweep over candidate values of $K$ and run weighted $k$-means on the coreset for each candidate. For each resulting partition, we compute the silhouette score and select the value with the highest score. In our experiments, we compute sensitivities once using $K_{max}=1000$, construct the coreset with respect to this value, and then sweep over candidate numbers of clusters in the range $K' \in [200,1000]$. The final value of $K$ is the one that maximizes the silhouette score on the coreset. Once this value is chosen, we run the final $k$-means clustering with that $K$ to define the semantic regions used by our method. Since the coreset is only used for model selection of the clustering granularity, and not as a replacement for the EMA mechanism itself, this preprocessing stage remains fully modular and can be replaced by other region-construction procedures in future work.
\setcounter{algorithm}{1}

\begin{algorithm}[t]
\caption{Region construction via coreset-based clustering}
\label{alg:region_construction}
\small
\begin{enumerate}
    \item Embed all prompts using the CLIP text encoder, obtaining
    $X=\{x_i\}_{i=1}^N \subset \mathbb{R}^d$.
    
    \item Run $k$-means++ with $K_{\max}$ centers on $X$ to obtain a bicriteria
    solution.

    \item Compute sensitivity upper bounds using the bicriteria solution,
    and sample a weighted coreset $(S,w)$.

    \item For each candidate number of clusters $k \in \mathcal{K}$:
    \begin{enumerate}
        \item run weighted $k$-means on $(S,w)$ with $k$ clusters;
        \item compute the corresponding silhouette score
        $\mathrm{Sil}_{(S,w)}(k)$.
    \end{enumerate}

    \item Select
    \[
    \hat{k}=\arg\max_{k\in\mathcal{K}} \mathrm{Sil}_{(S,w)}(k).
    \]

    \item Run $k$-means on the full embedding set with $k=\hat{k}$ and
    use the resulting partition to define the semantic regions.
\end{enumerate}
\end{algorithm}

\section{Ablation Studies}
\label{sec:ablation}

We analyze several design choices in ARIA: 
(i) the mapping used to convert region discrepancy scores into sampling probabilities, 
(ii) the number of regions used to partition the conditioning space, 
(iii) the discrepancy signal used for adaptive allocation, 
and (iv) the behavior of the EMA-based discrepancy tracker.
All ablations follow the same setup as Sec.~\ref{sec:exp} and use the BK-Base student model.

\paragraph{Sampling mapping.}
ARIA converts region discrepancy scores into sampling probabilities using a mapping $\mathcal{S}$.
We compare two choices: softmax ($p_k \propto \exp(D_k)$) and power-law ($p_k \propto D_k$).
Table~\ref{tab:ablation} shows that both mappings outperform uniform sampling.
Softmax yields the best FID and CLIP scores, suggesting moderate prioritization of difficult regions is more effective than overly concentrated allocation.



\paragraph{Number of regions.}
We study the impact of the number of regions $K$ used to partition the conditioning space.
The value $K=405$ is selected by the clustering procedure described in Sec.~D in the Supp. Matt.
As shown in Table~\ref{tab:ablation}, ARIA improves over uniform sampling across all tested $K$.
Moderate granularity ($K=405$) yields the best FID and IS, while larger partitions ($K=2000$) slightly improve CLIP.
Performance remains stable across region counts, indicating that ARIA does not require precise tuning of $K$.
\begin{table}[t]
\centering
\small
\setlength{\tabcolsep}{6pt}
\renewcommand{\arraystretch}{1.12}

\begin{tabular}{l l c c c}
\toprule
\textbf{Setting} & \textbf{Variant} & FID$\downarrow$ & IS$\uparrow$ & CLIP$\uparrow$ \\
\midrule

\textbf{Baseline}
& Uniform (RC) & 15.16 & 35.90 & 29.15 \\

\midrule

\multirow{2}{*}{Sampling mapping}
& Softmax & \textbf{13.82} & 36.52 & \textbf{29.28} \\
& Power-law & 14.70 & \textbf{36.57} & 29.25 \\

\midrule

\multirow{3}{*}{\# Regions $K$}
& $200$ & 14.51 & 35.94 & 29.18 \\
& $405$ & \textbf{13.82} & \textbf{36.52} & 29.28 \\
& $2000$ & 14.92 & 36.36 & \textbf{29.31} \\

\midrule

\multirow{3}{*}{Tracked discrepancy}
& $\mathcal{L}_{\mathrm{out}}$ & \textbf{13.82} & \textbf{36.52} & \textbf{29.28} \\
& $\mathcal{L}_{\mathrm{feat}}$ & 14.81 & 36.02 & 29.13 \\
& $\mathcal{L}_{\mathrm{out}} + \mathcal{L}_{\mathrm{feat}}$ & 14.46 & 36.32 & 29.18 \\

\bottomrule
\end{tabular}

\caption{Ablation study of ARIA on the BK-Base student model. We analyze the effect of the sampling mapping, the number of regions used to partition the conditioning space, and the discrepancy signal used for adaptive allocation.}
\label{tab:ablation}
\end{table}
\paragraph{Tracked discrepancy signal.}
ARIA reallocates training effort according to the magnitude of the teacher--student discrepancy.
We compare three allocation signals: the output-level loss $\mathcal{L}_{\mathrm{out}}$, the feature-level loss $\mathcal{L}_{\mathrm{feat}}$, and their combination.
As shown in Table~\ref{tab:ablation}, ARIA improves over uniform sampling for all objectives.
However, the largest gains occur when allocation is driven by $\mathcal{L}_{\mathrm{out}}$, which directly reflects mismatch in the predicted score function.
Feature-level discrepancies provide a weaker signal for adaptive allocation, as intermediate representations may remain aligned even when output-level mismatch persists.

\paragraph{EMA calibration of cluster difficulty.}
ARIA maintains cluster difficulty scores $D_s$ using an exponential moving average (EMA), which determines the sampling distribution over regions.
To verify that these online estimates reflect the true student--teacher discrepancy, we compare them against Monte Carlo estimates of the per-cluster loss at multiple training checkpoints on the C-Tiny architecture.
For each of the $S{=}405$ clusters, we sample $1000$ prompts, pair them with random latents, and add noise at timestep $t$ drawn from the same replacement distribution $p(t)$ used during training. We then evaluate the KD loss between the student and teacher and average the results to obtain the ground-truth score.

As shown in Fig.~\ref{fig:ema_calibration}, the EMA estimates closely track ground truth throughout training, with relative error below $4\%$ across checkpoints. Both quantities decay smoothly from ${\sim}2.4{\times}10^{-3}$ to ${\sim}1.0{\times}10^{-3}$, confirming that the lightweight tracker provides a reliable signal for focal sampling.

\begin{figure}[t]
    \centering
    \includegraphics[width=0.9\linewidth]{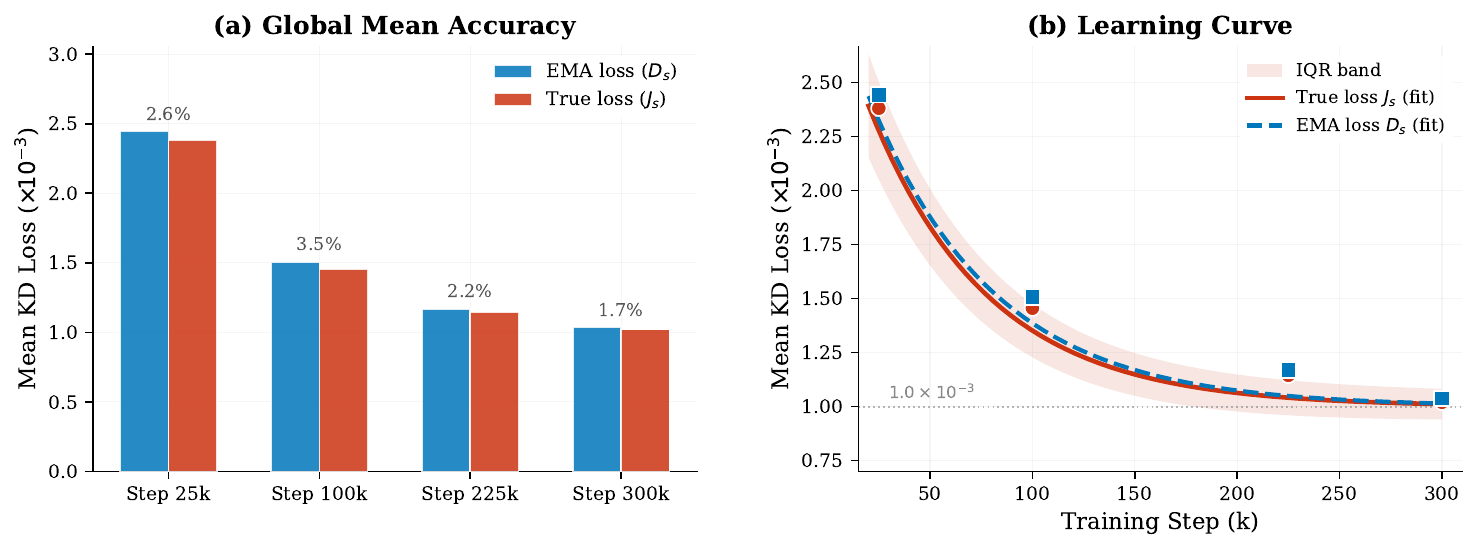}
    \caption{\textbf{EMA calibration study.}
(a) Comparison of the EMA-tracked mean cluster score $D_s$ and the Monte Carlo ground-truth loss $J_s$ at four checkpoints (25K, 100K, 225K, 300K). Percentages denote relative error. 
(b) Evolution of both quantities; the shaded band shows the interquartile range (IQR) of true per-cluster losses.
The EMA tracker closely follows the KD loss throughout training, with relative error below $4\%$.}
    \label{fig:ema_calibration}
\end{figure}

\end{document}